\documentclass[11pt, a4paper]{book}
\usepackage[hyperindex=true,
			bookmarks=true,
            pdftitle={Exploration in Feature Space for Reinforcement Learning}, 
            pdfauthor={Suraj Narayanan Sasikumar},
            colorlinks=true,
            pdfborder={0 0 0},
            pagebackref=false,
            citecolor=blue,
            linkcolor=blue,
            plainpages=false,
            pdfpagelabels,
            pagebackref=true,
            hyperfootnotes=false]{hyperref}
\usepackage[all]{hypcap}
\usepackage[palatino,final]{anuthesis}
\usepackage{afterpage}
\usepackage{graphicx}
\usepackage{thesis}
\usepackage{natbib}
\usepackage[normalem]{ulem}
\usepackage[table]{xcolor}
\usepackage{makeidx}
\usepackage{cleveref}
\usepackage[centerlast]{caption2}
\usepackage{float}
\usepackage[T1]{fontenc}
\usepackage[scaled]{beramono}

\usepackage{multirow}

\usepackage{amssymb}
\usepackage{amsthm}
\usepackage{bm}
\usepackage{mathtools}
\usepackage{epigraph}
\usepackage{ragged2e}
\usepackage[linesnumbered,ruled,vlined]{algorithm2e}

\makeatletter
\renewcommand{\algocf@Vline}[1]{
  \strut\par\nointerlineskip
  \algocf@push{\skiprule}
  \hbox{\bgroup\color{cyan}\vrule\egroup%
    \vtop{\algocf@push{\skiptext}
      \vtop{\algocf@addskiptotal #1}\bgroup\color{cyan}\Hlne\egroup}}\vskip\skiphlne
  \algocf@pop{\skiprule}
  \nointerlineskip}
\renewcommand{\algocf@Vsline}[1]{
  \strut\par\nointerlineskip
  \algocf@bblockcode%
  \algocf@push{\skiprule}
  \hbox{\bgroup\color{cyan}\vrule\egroup
    \vtop{\algocf@push{\skiptext}
      \vtop{\algocf@addskiptotal #1}}}
  \algocf@pop{\skiprule}
  \algocf@eblockcode%
}
\makeatother

\usepackage[export]{adjustbox}

\newtheorem{definition}{Definition}

\newtheorem{proposition}{Proposition}
\newtheorem{example}{Example}

\newcommand{\agent}{SARSA$(\lambda)$+$\bm\phi$-EB}

\renewcommand*{\backref}[1]{}
\renewcommand*{\backrefalt}[4]{
  \ifcase #1 %
  \or
    (cited on page #2)%
  \else
    (cited on pages #2)%
  \fi
}

\usepackage{booktabs}
\usepackage{relsize}
\usepackage{xspace}
\usepackage{subfigure}
\usepackage{listings}
\definecolor{tableheadcolor}{rgb}{0.8,0.8,1.0}
\definecolor{tablealtcolor}{rgb}{0.9,0.9,0.95}

\definecolor{todocolor}{rgb}{0.8,0.8,1.0}
\definecolor{fixcolor}{rgb}{1,0.8,0.8}
\definecolor{commentcolor}{rgb}{0.8,1.0,0.8}

\usepackage[color=todocolor, colorinlistoftodos]{todonotes}

\newcommand{\textjava}[1]{{\lstset{basicstyle=\ttfamily}\lstinline@#1@}}
\newcommand{\textjavafn}[1]{{\lstset{basicstyle=\footnotesize\ttfamily}\lstinline@#1@}}
\usepackage{setspace}
\usepackage{ifthen}

\long\def\sfootnote[#1]#2{\begingroup%
\def\thefootnote{\fnsymbol{footnote}}\footnote[#1]{#2}\endgroup}
\newcommand{\eg}{e.g., }
\newcommand{\ie}{i.e., }

\newcommand{\doi}[1]{\href{http://dx.doi.org/#1}{\nolinkurl{doi:#1}}}
\newcommand{\ignore}[1]{}

\title{Exploration in Feature Space for Reinforcement Learning}
\author{Suraj Narayanan Sasikumar}
\date{\today}

\renewcommand{\thepage}{\roman{page}}

\makeindex
\begin{document}
\pagestyle{empty}
\thispagestyle{empty}
\begin{titlepage}
  \enlargethispage{2cm}
  \begin{center}
    \makeatletter
    \Huge\textbf{\@title} \\[3.3cm]
    \huge\textbf{\@author} \\[3.3cm]
    \small Under the Supervision of \\
    \large \textbf{Professor\ Marcus Hutter}\\[3.3cm]
    \makeatother   
    \LARGE A thesis submitted for the degree of \\
    Master of Computing (Advanced) \\
    The Australian National University \\[0.5cm]
    \begin{figure}[H]
        \centering
        \includegraphics[scale=0.15]{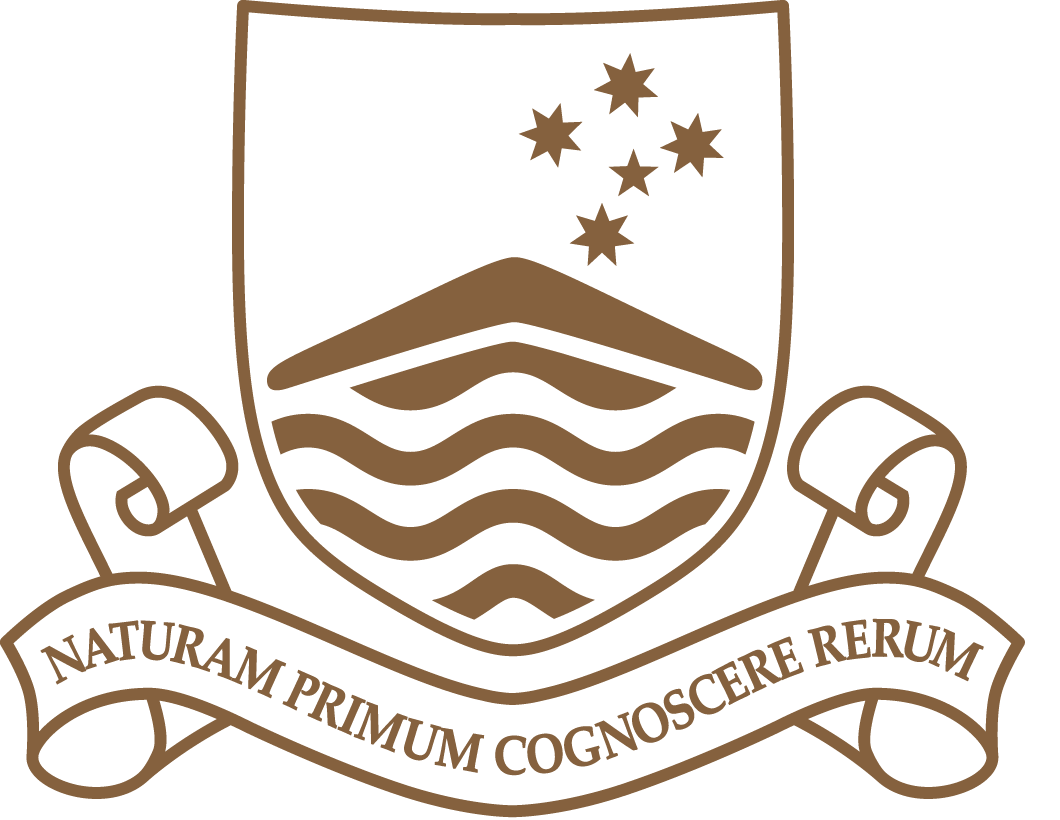}
    \end{figure}
    May 2017    
  \end{center}
\end{titlepage}

\vspace*{14cm}
\begin{center}
  \makeatletter
  \copyright\ \@author{} 2017
  \makeatother
\end{center}
\noindent
\begin{center}
  \footnotesize{~} 
\end{center}
\noindent

\newpage

\chapter*{Declaration}
\addcontentsline{toc}{chapter}{Declaration}

\begin{flushleft}
This thesis is an account of the research undertaken at the Research School of Computer Science, The Australian National University, Canberra, Australia.\vspace*{0.35cm}

The work presented in this thesis represents work conducted between July 2016 and May 2017. This work has been accepted for publication as "Count-Based Exploration in Feature Space for Reinforcement Learning" at the  \textit{26th International Joint Conference on Artificial Intelligence} (IJCAI) Melbourne, Australia, August 19-25 2017 \citep{Martin2017}.\vspace*{0.35cm}

The implementation of the algorithm, and the design of the infrastructure required to empirically evaluate the algorithm, are entirely my own original work. The design of the algorithm itself was done in collaboration with my colleague Jarryd Martin.\vspace*{0.35cm}

Except when acknowledged in the customary manner, the material presented in this thesis is, to the best of my knowledge, original and has not been submitted in whole or part for a degree in any university.
\end{flushleft}

\vspace*{2cm}
\hspace{8cm}\hrulefill\makeatother\par
\hspace{8cm}\makeatletter\@author\makeatother\par
\hspace{8cm}25 May 2017

\vspace*{2cm}

Supervisors:
\begin{itemize}
    \item Professor\ Marcus Hutter (Australian National University)
    \item Tom Everitt (Australian National University)
\end{itemize}

Convenor:
\begin{itemize}
    \item Professor\ John Slaney (Australian National University)
\end{itemize}

\cleardoublepage
\pagestyle{empty}
\vspace*{7cm}
\begin{center}
To my Friends and Family who literally made this possible.
\end{center}

\cleardoublepage
\pagestyle{empty}
\chapter*{Acknowledgments}
\addcontentsline{toc}{chapter}{Acknowledgments}

At the culmination of two years of hard work, I would like to take this opportunity to acknowledge the role they have played in my life.
 
\begin{itemize}
\item My supervisor Marcus Hutter, whose lectures inspired me to pursue Reinforcement Learning. Hearing him talk about Mathematics and AI with childlike enthusiasm has always been an inspiration to me. 
\item Tom Everitt, his steadfast and professional approach to supervising my thesis helped me a great deal in making this a success.
\item Jarryd Martin, my collaborator in this Project. Words cannot express the happiness that I feel, to have a friend in you. The days we toiled over broken theories and random agents were not for nothing. We did it\ldots\ and it could never have been achieved without you. Your determination, motivation and quite frankly the sheer energy is something that I aspire to.
\item John Aslanides, first a huge thank-you for talking your time to provide your comments and feedback on this thesis. Your rationality and cut-to-the-chase attitude. Your clarity of thought, work-ethic and passion for growth has been inspirational for me. 
\item Lulu Huang, for talking care of me throughout this whole endeavour, and being a wonderful roommate.
\item Boris Repasky, the man with infinite rigour. All the arguments and debates we've had throughout the year has only been for the better. Sina Eghbal, for being there through tough times as an unassuming friend. Darren Lawton, for forcing me to play Basketball. My AI Labmates, Sultan Javed, Owen Cameron, Arie Slobbe, Elliot Catt, for the motivation and support.
\item My father Sasikumar, sister Sumitha, for believing in me and being a vocal supporter of my decisions. My in-laws Rasmi, M Vasudevan, and Geetha, for supporting my decision to study. I miss you all very much.
\item My wife Ramya Vasudevan, her sacrifices and compromises are the reason I am able to do what I want to do, and I am forever indebted. None of this would even make sense without you in my life.
\item My mother Sudhamony whose guidance and sacrifices made the man I am. Through out my life she has been a constant source of moral guidance. I am a better person because of her.
\end{itemize}

\cleardoublepage
\pagestyle{headings}
\chapter*{Abstract}
\addcontentsline{toc}{chapter}{Abstract}
\vspace{-1em}

The infamous exploration-exploitation dilemma is one of the oldest and most important problems in reinforcement learning (RL). Deliberate and effective exploration is necessary for RL agents to succeed in most environments. However, until very recently even very sophisticated RL algorithms employed simple, undirected exploration strategies in large-scale RL tasks.

We introduce a new optimistic count-based exploration algorithm for RL that is feasible in high-dimensional MDPs. The success of RL algorithms in these domains depends crucially on generalization from limited training experience. Function approximation techniques enable RL agents to generalize in order to estimate the value of unvisited states, but at present few methods have achieved generalization about the agent's uncertainty regarding unvisited states. We present a new method for computing a generalized state visit-count, which allows the agent to estimate the uncertainty associated with any state. 

In contrast to existing exploration techniques, our $\bm\phi$-\textit{pseudocount} achieves generalization by exploiting the feature representation of the state space that is used for value function approximation. States that have less frequently observed features are deemed more uncertain. The resulting $\bm\phi$-\textit{Exploration-Bonus} algorithm rewards the agent for exploring in feature space rather than in the original state space. This method is simpler and less computationally expensive than some previous proposals, and achieves near state-of-the-art results on high-dimensional RL benchmarks. In particular, we report world-class results on several notoriously difficult Atari 2600 video games, including Montezuma's Revenge.

\cleardoublepage
\pagestyle{headings}
\addcontentsline{toc}{chapter}{Contents}
\markboth{Contents}{Contents}
\tableofcontents
\listoffigures
\addcontentsline{toc}{chapter}{List of Figures}
\listoftables
\addcontentsline{toc}{chapter}{List of Tables}
\listofalgorithms
\addcontentsline{toc}{chapter}{List of Algorithms}

\mainmatter

\chapter{Introduction}
\label{cha:intro}

\epigraph{\textit{`No great discovery was ever made without a bold guess.'}}{Isaac Newton}


\section{Reinforcement Learning}
\textit{Machine Learning} is a field in computer science that allows computers to dynamically generate novel algorithms that otherwise cannot be explicitly programmed. These algorithms, called \textit{hypotheses}, generalize patterns and regularities from observed real-world data using statistical techniques~\citep{Bishop2007}. \textit{Reinforcement Learning} (RL) is a field of machine learning that deals with optimal sequential decision making in an unknown environment with no explicitly labelled training data. The RL framework is one of the fundamental models that best describes how intelligent beings interact with their world to achieve a \textit{goal}. An RL algorithm is given agency to interact with its surroundings, and is aptly called an \textit{agent}. The world with which the agent interacts is called its \textit{environment}. The \textit{unsupervised} nature of RL algorithms means that the agent has to develop a \textit{policy} for acting in an unknown environment by trial-and-error~\citep{Sutton1998}. In every such interaction the agent performs an action on the environment and receives a \textit{percept}. The percept consists of the current configuration of the environment, called \textit{state}, and a scalar feedback signal, called \textit{reward}. The reward signal indicates how good the sequence of \textit{actions} of the agent was. The \textit{goal} of an RL agent is based on the concept of the \textit{reward hypothesis}:
\begin{definition}[Reward Hypothesis]~\citep{Sutton1999}
    Any notion of a \textit{goal} or \textit{purpose} of an intelligent agent can be described as the maximization of expected cumulative reward.
\end{definition}
The existence of an extrinsic feedback signal makes RL algorithms also somewhat supervised in nature - thus RL algorithms are in some sense both supervised and unsupervised~\citep{Barto2004}.

As an example, consider an agent playing a car racing game in which the goal is to reach the finish line as soon as possible. To model the goal as a cumulative reward maximization problem, we give the agent a negative reward every time step, thereby incentivizing the agent to reach the finish line as quickly as possible. This example illustrates how an objective can be modelled as the maximization of expected cumulative reward, and the \textit{goal} of an agent as a sequence of actions that achieves it. The interaction between agent and environment continues until the agent converges to an optimal sequence of actions for each state in the environment. This interaction is called the agent-environment interaction cycle, as illustrated in \cref{fig:agent_env}. Each iteration of the interaction is called a time-step, often denoted by the subscript $t$ to distinguish states, actions, and percepts between time-steps.

\begin{figure}[H]
  \label{fig:agent_env}
  \centering
  \includegraphics[scale=0.7]{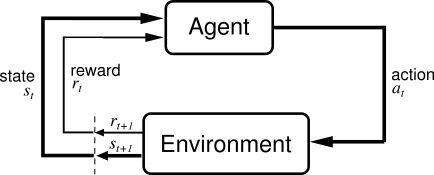}
  \caption{The agent-environment interaction cycle~\citep{Sutton1998}}
\end{figure}

\section{The Exploration/Exploitation Dilemma}\label{dilemma}
In an online decision-making setting such as the reinforcement learning problem, an agent is faced with two choices - \textit{explore} or \textit{exploit}. The term \textit{exploration} in an active learning system is defined as the process of deliberately taking a non-greedy action with the sole aim of gathering more information about the environment. Exploration plays a fundamental role in reinforcement learning algorithms. It is born out of the notion that an optimal long-term policy might involve short-term sacrifices. Alternatively, \textit{exploitation} is the act of taking the best possible action given the current information about the environment. A central challenge in reinforcement learning is to find the sweet spot between exploration and exploitation, \ie to figure out when to explore and when to exploit. This problem is known as the \textit{exploration-exploitation dilemma}. 

At present there are a number of provably efficient exploration methods that are effective in environments with low-dimensional state-action spaces. Most of the exploration algorithms which enjoy strong theoretical guarantees implement the so-called "Optimism in the Face of Uncertainty" (OFU) principle. This heuristic encourages the agent to be optimistic about the reward it might attain in less explored parts of the environment. The agent seeks out states with higher associated uncertainty, and in doing so reduces its uncertainty in a very efficient way. Many algorithms that implement this heuristic do so by adding an exploration bonus to the agent's reward signal. This bonus is usually a function of a state visit-count; the agent receives higher exploration bonuses for exploring less frequently visited states (about which it is less certain).

Unfortunately, these algorithms do not scale well to high-dimensional environments. In these domains, the agent can only visit a small portion of the state space while it is training. The visit-count for most states is always zero, even after training is finished. Nearly all states will be assigned the same exploration bonus throughout training. This renders the bonus useless as a tool for efficient exploration. All unvisited states appear to the agent as equally uncertain. This problem arises because these count-based OFU algorithms fail to generalise the agent's uncertainty from one context to another. Even if an unvisited state has very similar features to a frequently visited one, the agent will treat the former as a complete unknown. Consequently even the sophisticated algorithms that are suitable for the high-dimensional setting  -- e.g. those that use deep neural networks for policy evaluation -- tend to use simple, inefficient exploration strategies. 

Success in the high-dimensional setting demands that the agent represent the state space in a way that allows generalisation about uncertainty. This sort of generalisation would allow that the agent's uncertainty be lower for states with familiar features, and higher for states with novel features, even if those exact states haven't been visited. What we require, then, is an efficient method for computing a suitable similarity measure for states. That is the key challenge addressed in this thesis.

\section{Summary of Contributions}
This thesis presents a new count-based exploration algorithm that is feasible in environments with large state-action spaces. It can be combined with any value-based RL algorithm that uses linear function approximation (LFA). The principal contribution is a new method for computing generalised visit-counts. Following \cite{Bellemare2016}, we construct a visit-density model in order to measure the similarity between states. Our approach departs from theirs in that we do not construct our density model over the raw state space. Instead, we exploit the feature map that is used for value function approximation, and construct a density model over the transformed feature space. This model assigns higher probability to state feature vectors that share features with visited states. Generalised visit-counts are then computed from these probabilities; states with frequently observed features are assigned higher counts. These counts serve as a measure of the uncertainty associated with a state. Exploration bonuses are then computed from these counts in order to encourage the agent to visit regions of the state-space with less familiar features.

Our density model can be trivially derived from any feature map used for LFA, regardless of the application domain, and requires little or no additional design. In contrast to existing algorithms, there is no need to perform a special dimensionality reduction of the state space in order to compute our generalised visit-counts. Our method uses the same lower-dimensional feature representation to estimate value and to estimate uncertainty. This makes it simpler to implement and less computationally expensive than some existing proposals. Our evaluation demonstrates that this simple approach achieves near state-of-the-art performance on high-dimensional RL benchmarks.

\chapter{Background and Related Work}
\label{cha:background}

\epigraph{\textit{`Each night, when I go to sleep, I die. And the next morning, when I wake up, I am reborn.'}}{Mahatma Gandhi}

In this chapter we give a formal treatment of the RL problem. We then provide a taxonomy of RL algorithms and the challenges posed by classical RL algorithms. Further, we talk about relevant research findings on how to solve these challenges.

\section{Classical Reinforcement Learning}
\label{sec:rl_primer}
In Classical RL (CRL), the environment is assumed to be fully observable, ergodic, and every state has the \textit{Markov property}. The branch of reinforcement learning where these assumptions are lifted is called General Reinforcement Learning (GRL)~\citep{Hutter2005}.
\begin{definition}[Markov property]
    Future states are only dependent on the current states and action, and are independent of the history of percepts. Formally,
    \[P(s_{t+1}=s',r_{t+1}=r \mid s_t,a_t,r_t,s_{t-1},a_{t-1},\ldots,r_1,s_0,a_0) = P(s_{t+1}=s',r_{t+1}=r \mid s_t,a_t)\]
    for all $s',r$, and histories $s_t,a_t,r_t,s_{t-1},a_{t-1},\ldots,r_1,s_0,a_0$
\end{definition}
A \textit{Markov Decision Process} (MDP) captures the above assumptions about the environment, and so in the CRL context the environment is modelled as an MDP~\citep{Puterman1994}. Thus, the CRL problem now reduces to the problem of finding an optimal policy for an unknown MDP\@.
\begin{definition}[Markov Decision Process]
A Markov Decision Process is a Tuple $\langle\mathcal{S},\mathcal{A},\mathcal{P},\mathcal{R},\gamma\rangle$ representative of a fully-observable environment in which all states are Markov.
\begin{itemize}
    \item $\mathcal{S}$ is a finite set of states
    \item $\mathcal{A}$ is a finite set of actions
    \item $\mathcal{P}_{ss'}^a = \mathbb{P}[s_{t+1}=s'\mid s_t=s, a_t=a]$ are the transition probabilities
    \item $\mathcal{R}_{ss'}^a = \mathbb{E}[r_{t+1}\mid s_t=s, a_t=a,  s_{t+1}=s']$ is the expected value of the reward resulting from the transition $s,a,s'$
    \item $\gamma$ is the discount factor which weights the relative importance of immediate rewards to future rewards.
\end{itemize}
\end{definition}
If the dynamics (transition and reward distributions) of the MDP are known, then we can use dynamic programming methods to directly plan on the MDP to find an optimal policy. In the RL context, in which the system dynamics are unknown, we have to use iterative RL algorithms such as TD-learning~\citep{Sutton1988} to find a good policy asymptotically.\footnote{Asymptotic analysis is one of the few theoretical tools we have to analyse RL algorithms in a domain-agnostic way.}
\begin{definition}[Policy]
    A policy may be deterministic or stochastic. A deterministic policy is a mapping from the states to actions.
    \[\pi: \mathcal{S} \rightarrow \mathcal{A}\]
    A stochastic policy is a probability distribution over the set of actions given a state.
    \[\pi(a\mid s_t=s)\]
\end{definition}

\subsubsection*{Value}
The most common way to characterize the quality of a given policy is to define a function that computes how valuable it is to follow the policy from a given state (or state-action pair). This notion of value is expressed in terms of future rewards the agent could expect, if it had chosen to follow the given policy.


\begin{definition}[State-Value Function]
    The state-value function, $V^\pi(s)$ is a mapping from states to $\mathbb{R}$. The value of a state $s\in\mathcal{S}$ under policy $\pi$ is the expected discounted cumulative reward given that the agent starts in state $s$ and follows policy $\pi$ thereafter.
    \[V^\pi(s) = \mathbb{E}_\pi\Bigg[\sum_{k=0}^{\infty} \gamma^k r_{t+k+1}\mid s_t=s\Bigg]\]
\end{definition}
\begin{definition}[Action-Value Function]
    The action-value function, $Q^\pi(s,a)$ is a mapping from state-action pairs to $\mathbb{R}$. The action-value of the state-action pair $(s,a)$ under policy $\pi$ is the expected discounted cumulative reward given that the agent starts in state $s$, takes action $a$, and follows policy $\pi$ thereafter.
    \[Q^\pi(s,a) = \mathbb{E}_\pi\Bigg[\sum_{k=0}^{\infty} \gamma^k r_{t+k+1}\mid s_t=s,a_t=a\Bigg]\]
\end{definition}

\subsubsection*{Bellman Equations}
Bellman equations form the basis for how to compute, approximate, and learn value functions in the RL setup~\citep{Sutton1998}. They arise naturally from the structure of an MDP by capturing the recursive relationship between the value of a state and the value of its successor states. The two Bellman equations for the state-values and action-values can be defined as follows.
\begin{definition}[Bellman Equation for state-value function of an MDP]
    \[V^\pi(s) = \sum_{a\in\mathcal{A}}\pi(a\mid s)\Bigg[\sum_{s'\in\mathcal{S}}\mathcal{P}_{ss'}^a\Big[R_{ss'}^a + \gamma V^\pi(s')\Big]\Bigg]\]
\end{definition}
\begin{definition}[Bellman Equation for action-value function of an MDP]
    \[Q^\pi(s, a) = \sum_{s'\in \mathcal{S}}\mathcal{P}_{ss'}^a\Big[\mathcal{R}_{ss'}^a + \gamma\sum_{a'\in\mathcal{A}}\pi(a'\mid s')Q^\pi(s',a')\Big]\]
\end{definition}
We can now use the value function to define a partial ordering over policies. A policy is said to be better than another when the expected return of one policy is greater than or equal to the other for all states. Formally, $\pi \succsim \pi' \iff V^\pi(s) \geq V^{\pi'}(s)\quad \forall s\in \mathcal{S}$. From the imposed partial ordering it has been shown that there exists at least one policy, $\pi^*$, such that $\pi^* \succsim \pi$ for all policies $\pi$, although it might not be unique~\citep{Bertsekas1996}. The Bellman Optimality Equations provide a mathematical framework for talking about the optimal policy just by replacing the sum over actions with a $\max$ operator. Intuitively, this represents a policy that is greedy with respect to the value of its successor states.
\begin{definition}[Bellman Optimality Equation for state-values]
    \[V^{\pi^*}(s)\equiv V^*(s) = \max_{a\in\mathcal{A}}\sum_{s'\in\mathcal{S}}\mathcal{P}_{ss'}^a\Big[\mathcal{R}_{ss'}^a + \gamma V^*(s')\Big]\]
\end{definition}
\begin{definition}[Bellman Optimality Equation for action-values]
    \[Q^{\pi^*}(s,a)\equiv Q^*(s,a) = \sum_{s'\in\mathcal{S}}\mathcal{P}_{ss'}^a\Big[\mathcal{R}_{ss'}^a + \gamma \max_{a'\in\mathcal{A}}Q^*(s',a')\Big]\]
\end{definition}
For finite MDPs with known environment dynamics, the Bellman Optimality Equations have a unique solution. Unfortunately in the RL setup we deal with an unknown MDP\@. Thus, almost all of the RL algorithms approximate the Bellman Optimality Equations for an unknown MDP and try to iteratively find an optimal policy asymptotically.

\subsection{Reinforcement Learning Algorithms}
\label{sec:rl_algs}
The fundamental difference between an RL problem and a \textit{planning} problem is the knowledge of the environment dynamics. In a planning problem the model of the environment is already known and the problem boils down to finding an optimal policy in the environment. In an RL problem, the agent is dropped into an unknown environment the dynamics of which is unknown. This makes reinforcement learning a hard problem. This distinction gives rise to two categories of RL algorithms, namely \textit{model-based} and \textit{model-free}.

The class of algorithms that learns the model of the environment, and then does planning within the learned model are called \textit{model-based} RL algorithms. These algorithms learn the transition probabilities ($\mathcal{P}_{ss'}^a$) and reward functions ($\mathcal{R}_{ss'}^a$) of the MDP by iteratively simulating the environment and updating the simulation to better represent the true environment. This approach to solve unknown MDP's is computationally intensive, especially in large or continuous problems. Value iteration and policy iteration are two dynamic programming algorithms that have a planning-based approach to the RL problem. On the other hand, \textit{model-free} algorithms directly learn the optimal policy using an intermediary quantity (usually the value-function).

\subsubsection*{Generalized Policy Iteration (GPI)}
The overarching theme of almost all value-function based CRL algorithms is the back-and-forth between two interacting processes, \textit{prediction} and \textit{control}, eventually resulting in convergence. \textit{Prediction} refers to policy-evaluation where the value-function is estimated for the current policy. \textit{Control} on the other hand aims to find a policy that is greedy with respect to the current value-function (state-value or action-value).
\begin{figure}[H]
\label{fig:gpi_conv}
  \centering
  \includegraphics[scale=0.7]{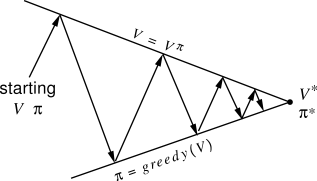}
  \caption{Generalized Policy Iteration~\citep{Sutton1998}}
\end{figure}

The `Prediction and Control' process converges when it produces no significant change, that is, the value-function is consistent with the current policy and the policy is greedy with respect to the current value-function.

\subsubsection*{Temporal Difference Learning}

Temporal Difference learning (TD learning) is a common RL algorithm; it is a model-free algorithm that combines Monte Carlo methods with the ideas from dynamic programming. TD learning allows the agent to directly learn from its experience of the environment. Following the GPI theme, we need a strategy for prediction and control. In TD prediction we use the sampling of Monte Carlo methods and bootstrapping (updating from an existing estimate) of DP algorithms to estimate the current value-function.

\begin{definition}[Update formula for state-value function]
    \[V(s_t) \leftarrow V(s_t) + \alpha[r_{t+1} +\gamma V(s_{t+1}) - V(s_t)]\]
\end{definition}
TD(0) is a TD learning algorithm that updates state-values after each time-step, so the learning process is fast and on-line. The target for the TD(0) update formula uses the existing estimate of $V(s_{t+1})$, hence we say the algorithm bootstraps.

As the agent interacts with the environment more, TD learning is able to generate a better estimate of the value-functions. In the limit, if each state (or state-action pair) is visited infinitely often with some additional constraints on the learning rates, convergence to the true value-function is guaranteed~\citep{Bertsekas1996}.

In TD control we want to optimize the value-function of an unknown environment. There are two classes of policy control methods, namely, \textit{on-policy} and \textit{off-policy}. On-policy control uses the policy derived from the current value-function estimate to update the future estimates. Alternatively, off-policy control uses a policy that is greedy with respect to the current value-function to estimate future value-functions. $SARSA$ (on-policy) and $Q$-learning (off-policy) are two popular TD control algorithms that are known to learn an MDP asymptotically~\citep{Sutton1998,Watkins1992}.

The important concept of why we are able to do model-free TD control lies in the fact that we use (state,action)-value functions instead of state-value functions.
\begin{definition}[Greedy policy control] Policy improvement is done by considering a new policy, $\pi'$, which is greedy with respect to the current value-function.
    \[\pi'(s) = \arg\max_{a\in \mathcal{A}}\mathcal{R}_s^a + \mathcal{P}_{ss'}^a V(s')\qquad(\textit{Greedy w.r.t. state-value function})\]
    \[\pi'(s) = \arg\max_{a\in \mathcal{A}} Q(s,a)\qquad\qquad\quad\;\;\;(\textit{Greedy w.r.t. action-value function})\]
\end{definition}

From the above policy improvement equations we can see that in order to be greedy with respect to the state-value function, we require the model of the MDP\@. In contrast, if the policy is greedy with respect to the action-value function, the model dynamics of the MDP is not needed, and hence, it is \textit{model-free}. Thus, optimizing action-value functions to learn the optimal policy is at the heart of all model-free TD control algorithms.

\subsubsection*{Challenges and Drawbacks}
All the Classical Reinforcement Learning algorithms that we discussed above can be categorized as \textit{tabular} algorithms. That is, the algorithms use a table data structure to associate each state (or state-action pair) with its current value estimate. As the agent interacts with the environment and gains experience, the table values are updated with better estimates of its value.

The main drawback of such a method is that it scales poorly. When the state-space is very large or continuous, the fundamental requirement that the agent visits each state (or state-action pair) multiple times (or infinitely often) is not satisfied; states are at most visited once. An agent following a policy derived from these value estimates would do no better than a random policy. Moreover, the table size grows with the number of states, making storage infeasible for problems with large/continuous state-space.

A common approach to solving this problem is to find a way to \textit{generalize} the value-function from the limited experience of the agent~\citep{Sutton1998}. That is, we want to approximate the value-function for an unseen state (or state-action pair) from the example values it has observed so far. \textit{Function approximation} is a generalization technique that does exactly this; it takes in observed values of a desired function and attempts to generalize an approximation of the function.

\subsection{Function Approximation}
\label{sec:fa}
Function approximation (FA) is an instance of supervised learning~\citep{Sutton1998}. It is viewed as a class of techniques used to approximate functions by using example values of the desired function. In the RL context tabular methods become infeasible in large or continuous state spaces. This challenge is mitigated by employing FA techniques to predict the value-function at unseen states. However, not all FA methods are applicable to the RL setting. We require a training method which can learn efficiently from on-line, non-i.i.d.\ data, and also handle non-stationary target functions. The following are some of the function approximators that are used in the RL context.


\begin{itemize}
    \item Gradient-Descent Methods
    \begin{itemize}
        \item Artificial Neural Networks
        \item Linear Combination of Features
    \end{itemize}
    \item State Aggregation
    \begin{itemize}
        \item k-Nearest Neighbors
        \item Soft Aggregation
    \end{itemize}
\end{itemize}

State aggregation is a method of generalizing function approximation in which states are grouped based on a criterion and then value is estimated as an attribute of the group. When a state is re-visited the value corresponding to the state's group gets updated.

Linear combination of features, also known as \textit{Linear Function Approximation} (LFA), is essentially a linear mapping from the state space (of dimensionality $D$) to a feature space of dimension $M$, where often $M<D$. Each basis function of the feature space is a mapping from the state space to a real-valued number that represents some feature of the state-space.
\begin{definition}[Linear-Approximate state(action)-value function] The approximate state-value function of a state $s\in \mathcal{S}$ under a policy $\pi$ is given by:
\begin{align*}
    \hat V^\pi(s) &= \bm{\theta}^T \bm{\phi}(s)=\sum_{i=1}^{M}\theta_i\phi_i(s)\\
    \hat Q^\pi(s, a) &= \bm{\theta}^T \bm{\phi}(s, a)=\sum_{i=1}^{M}\theta_i\phi_i(s, a)
\end{align*}

Where $\bm{\phi}:\mathcal{S}(\times \mathcal{A})\rightarrow \mathbb{R}^M$, is a feature map, and $\bm{\theta}\in \mathbb{R}^M$ is the parameter vector.
\end{definition}

LFA has sound theoretical guarantees and also is very efficient in terms of both data and computation~\citep{Sutton1998}, making it a good candidate for the implementation of our algorithm.

As mentioned previously, FA can be regarded as a technique to develop a generalization regarding value. In order to have a good capacity to generalize, a function approximator must have relevant data about the state-space. Consider a pathological case in which the agent does not explore at all: as a result the only data available for FA would be concentrated in one region of the state space. This results in the estimation of values of unseen states being highly biased. In order to avoid this problem we have to make sure that the agent visits most regions of the state-space, that is, the agent has to explore the state-space efficiently. The main goal of this thesis is to address the problem of how to explore efficiently in large state-spaces.

\section{Exploration Strategies for Reinforcement Learning}
In \cref{dilemma} we described the \textit{exploration/exploitation dilemma}, which is a fundamental problem in RL\@. All exploration strategies attempt to manage the trade-off between these two often opposed objectives. The simplest and most widely-used exploration strategy is known as $\epsilon$-greedy. At each time-step $t$ the agent chooses a greedy action with probability $1-\epsilon$ and with $\epsilon$ probability the agent chooses a completely random action. To ensure that the policy converges to the optimal policy it has to satisfy the \textit{GLIE} assumptions ~\citep{Singh2000a}: 
\begin{definition}[Greedy in the Limit with Infinite Exploration] A policy is \textit{GLIE} if it satisfies the following two assumptions. 

\begin{itemize}
    \item Each action is taken infinitely often in every state that is visited infinitely often,
    \[\lim_{t\rightarrow\infty} N_t(s,a) = \infty\]
    Where $N_t(s,a)$ is the number of times action $a$ has been chosen in state $s$ up-to time-step $t$.
    \item In the limit, the learning policy is greedy with respect to the Q-value function with probability $1$.
    \[\lim_{t\rightarrow\infty}\pi_t(a\mid s) = 1, \textit{when, } a = \arg\max_{a'\in\mathcal{A}} Q_t(s,a')\]
\end{itemize}
\end{definition}

For example, $\epsilon$-greedy satisfies the GLIE assumptions when $\epsilon$ is annealed to zero. A common way to do this is by setting $\epsilon_t\propto 1/t$.

In small, finite MDPs $\epsilon$-greedy satisfies the GLIE assumptions, but when the state-action space is large/continuous the first GLIE assumption is violated and hence the convergence guarantee is lost. $\epsilon$-greedy is a na\"{i}ve approach to solving the exploration problem, but we still use it in large MDPs because of its low resource requirements when compared with  alternatives ~\citep{Bellemare2016}. In this thesis we propose a novel exploration strategy that improves upon $\epsilon$-greedy, and provides state-of-the-art results in large problems with low computational overhead. 

We now provide an exposition of various explorations strategies, their foundational principles, and an analysis of recent breakthroughs in the field of exploration.

\subsection{Taxonomy of Exploration Strategies}
The exploration-exploitation dilemma is still an open problem, but researchers have made significant inroads into understanding the nature of the problem. Sebastian Thrun classified exploration techniques into two families of exploration schemes, \textit{directed} and \textit{undirected} ~\citep{Thrun1992a}. Undirected exploration strategies do not use any information from the environment to make an informed exploratory action; they predominantly rely on randomness to do exploration. Softmax methods and $\epsilon$-greedy are examples of undirected exploration techniques.  The \textit{softmax} action is sampled from the \textit{Boltzman distribution}

$$\text{Boltz}_s(a) = \frac{\exp(Q(s,a))}{\sum_{a'\in\mathcal{A}}\exp(Q(s,a'))}.$$

On the other hand, directed exploration strategies use the knowledge about the learning process to form an exploration-specific \textit{heuristic} for action selection. This heuristic directs the agent to take those actions that maximizes the information gain about the environment. The exploration algorithm introduced in this thesis falls into the category of directed exploration algorithms. In order to put it into context, we first present an overview of the existing directed exploration strategies used in the literature.


\section{The Optimism in the Face of Uncertainty Principle}
In the following chapter we present our directed exploration method, which implements the principle of "Optimism in the Face of Uncertainty" (OFU) as a heuristic for exploration. In this section we review existing work  on the OFU heuristic. The principle is succinctly captured in \cite{Osband2016}:

\begin{quote}
"When at a state, the agent assigns to each action an optimistically biased while statistically plausible estimate of future value and selects the action with the greatest estimate."
\end{quote}

OFU is a heuristic to direct exploratory actions. OFU directs the agent to take actions which have more uncertain value estimates. Instead of greedily taking the action that has the highest estimated value, that agent is encouraged to take actions which have a high \emph{probability} of being optimal. To see that an apparently suboptimal action may indeed have a high probability of being optimal, let us take an example. Suppose that the agent has taken an action $a\in{\mathcal{A}}$ very often from a particular state $s\in\mathcal{S}$, and suppose that $a$ also currently has the highest value-estimate $\hat{Q}^{\pi}(s,a)$ among the available actions. Now consider an alternative action $a'\in \mathcal{A}$ that has only been tried once from the state $s$, and suppose that the reward received was lower than $\hat{Q}^{\pi}(s,a)$. Action $a$ has higher estimated value, but having tried it many times, the agent's uncertainty about its value is quite low. In contrast, the uncertainty about the value of the alternative action $\hat{Q}^{\pi}(s,a')$ is very high, since it has been taken so rarely. Thus, while the current estimate $\hat{Q}^{\pi}(s,a')$ may be lower than $\hat{Q}^{\pi}(s,a)$, there is a good chance that the agent was unlucky when taking $a'$ the first time, and that the true action-value $Q^{\pi}(s,a')$ is much higher than both estimates. Thus it may be that $a'$ has a higher probability of being the optimal action than does $a$, especially if their estimated values are quite close. The OFU heuristic would bias the agent toward taking action $a'$ instead of the greedy action $a$. An agent following this heuristic will behave as if it is optimistic about action $a'$, or more precisely, about its true action-value $Q^{\pi}(s,a')$. This optimism drives the agent to explore regions of the environment about which it is more uncertain.

\subsection{OFU using Count-Based Exploration Bonuses}
Most of the exploration algorithms that enjoy strong theoretical efficiency guarantees, implement the OFU heuristic. Many do so by augmenting the estimated value of a state(-action pair) with an exploration bonus that quantifies the uncertainty in that value estimate. An agent which acts greedily with respect to this augmented value function will be biased to take actions with higher associated uncertainty. Most of these algorithms are \emph{tabular} and \emph{count-based} in that they compute their exploration bonuses using a table of state(-action) visit-counts. The visit-count serves as an approximate measure of the uncertainty associated with a state(-action), because more novel state(-action) pairs will have lower visit-counts. State(-actions) with lower visit counts are assigned higher exploration bonuses. This drives the agent to behave optimistically and explore less frequently visited regions of the environment, which may yet prove to have higher value than familiar regions. Moreover, even if those regions turn out to yield little reward when explored, the agent will have greatly reduced its uncertainty about those regions. Indeed, the reduction in uncertainty would be much smaller if the agent were to take an action that had already been tried many times. The OFU heuristic is therefore a win-win approach for the agent. OFU algorithms are more efficient than undirected exploration strategies like $\epsilon$-greedy because the agent avoids actions that yield neither large rewards nor large reductions in uncertainty \citep{Osband2016a}. 

\subsection{Tabular Count-based Exploration Algorithms}
One of the best known OFU methods is the UCB1 bandit algorithm, which selects an action that maximises an upper confidence bound $\hat{Q}_{t}(a)+\sqrt{\frac{2\log t}{N(a)}}$, where $\hat{Q}_{t}(a)$ is the estimated mean reward and $N(a)$ is the visit-count \citep{Lai1985}. The dependence of the bonus term on the inverse square-root of the visit-count is justified using Chernoff bounds. In the MDP setting, the tabular OFU algorithm most closely resembling our method is Model-Based Interval Estimation with Exploration Bonuses (MBIE-EB) \citep{Strehl2008}.\footnote{To the best of our knowledge, the first work to use exploration bonuses in the MDP setting was the Dyna-$Q$+ algorithm, in which the bonus is a function of the recency of visits to a state, rather than the visit-count \citep{Sutton1990}} Empirical estimates $\hat{\mathcal{P}}$ and $\hat{\mathcal{R}}$ of the transition and reward functions are maintained, and $\mathcal{\hat{R}}(s,a)$ is augmented with a bonus term $\frac{\beta}{\sqrt{N(s,a)}}$, where $N(s,a)$ is the state-action visit-count, and $\beta\in\mathbb{R}$ is a theoretically derived constant. The Bellman optimality equation for the augmented action-value function is $$\tilde{Q}^{\pi}(s,a)=\mathcal{\hat{R}}(s,a)+\frac{\beta}{\sqrt{N(s,a)}}+\ \gamma\sum_{s'}\hat{\mathcal{P}}(s'\mid s,a)\max_{a'\in\mathcal{A}}\tilde{Q}^{\pi}(s',a')$$ Here the dependence of the bonus on the inverse square-root of the visit-count is provably optimal \citep{Kolter2009}. This equation can be solved using any MDP solution method.

While tabular OFU algorithms perform well in practice on small MDPs \citep{Strehl2004}, their \textit{sample complexity} becomes prohibitive for larger problems \citep{Bellemare2016}. The sample complexity of an algorithm is a bound on the number of timesteps at which the agent is not taking an $\epsilon$-optimal action with high probability \citep{Kakade2003}. Loosely speaking, it measures the amount of experience the agent must have before one can be confident it is basically performing optimally. MBIE-EB, for example, has a sample complexity bound of $\tilde{O}\big(\frac{\left|\mathcal{S}\right|^{2}\left|\mathcal{A}\right|}{\epsilon^{3}(1-\gamma)^{6}}\big)$. In the high-dimensional setting -- where the agent cannot hope to visit every state during training -- this bound offers no guarantee that the trained agent will perform well. 
The prohibitive complexity of these tabular OFU algorithms is due in part to the fact that a table of visit-counts is not useful if the state-action space is too large. Since the agent will only visit a small fraction of that space, the visit-count for most states will always be zero. These algorithms are therefore unable to usefully compare the novelty of two unvisited states. All unvisited states have the same visit-count, and hence the same exploration bonus. The optimistic agent will treat them all as equally novel and equally appealing.

\subsection{Generalized Visit-counts for Exploration in Large MDPs}
\label{sec:large_mdps}
Tabular OFU algorithms fail on high-dimensional problems because they do not allow for generalization across the state space regarding uncertainty. Every unvisited state is treated as entirely novel, regardless of any similarity between the unvisited states and the visited states in the history. In order to explore efficiently in large domains, the agent must be able to make use of the fact that some unvisited states share many features with visited states, while others share very few. If an unvisited state has almost exactly the same features as a very frequently visited one, then it should not be considered to be as uncertain as a state with unfamiliar features. An effective OFU method for these problems would not just encourage the agent to visit unvisited states, but rather would drive the agent to visit states with novel or uncommon features. We discuss this issue further in section \cref{sec:conf_novelty}.

Several very recent extensions of count-based exploration methods have achieved this sort of generalisation regarding uncertainty, and have produced impressive results on high-dimensional RL benchmarks. These algorithms closely resemble MBIE-EB, but they substitute the state-action visit-count for a \emph{generalised visit-count} which quantifies the similarity of a state to previously visited states. \cite{Bellemare2016} construct a Context Tree Switching (CTS) density model over the state space such that higher probability is assigned to states that are more similar to visited states \citep{Veness2012}. A state pseudocount is then derived from this density. A subsequent extension of this work replaces the CTS density model with a neural network \citep{Ostrovski2017}. Another recent proposal uses locality sensitive hashing (LSH) to cluster similar states, and the number of visited states in a cluster serves as a generalised visit-count \citep{Tang2016}. As in the MBIE-EB algorithm, these counts are used to compute exploration bonuses. These three algorithms outperform random strategies, and are currently the leading exploration methods in large discrete domains where exploration is hard.

Before presenting our optimistic count-based exploration method in the following chapter, we now briefly canvas two alternative frameworks for directed exploration, and discuss their limitations.

\section{Bayes-Adaptive RL}

In the Bayesian approach to model-based reinforcement learning, we maintain a posterior distribution over the possible models of the environment given the experience of the agent ~\citep{Dearden1998}. Bayesian inference is used to update the posterior with new information as the agent interacts with the environment, and also to incorporate the agent's prior distribution over the transition models.

Since the posterior is maintained over all possible models we can now talk about the uncertainty pertaining to what is the best action to take. This uncertainty is modelled as a Markov Decision Process defined over a set of \textit{hyper-states}. A hyper-state acts as an information state which summarizes the information accumulated so far. This augmented MDP, often referred to as the Bayes-Adaptive MDP (BAMDP), can be solved with standard RL algorithms ~\citep{Duff2002}. In this framework an agent acting greedily in the BAMDP whilst updating the posterior acts optimally (according to its prior belief) in the original MDP. The Bayes-optimal policy for the unknown environment is the optimal policy of the BAMDP, thereby providing an elegant solution to the exploration-exploitation trade-off.

Unfortunately, the cardinality of the hyper-states grows exponentially with the planning horizon thereby rendering exact solution to the BAMDP computationally intractable for large problems ~\citep{Duff2002}. 

\section{Intrinsic Motivation}
The final directed exploration heuristic that we discuss is born out of the so-called \emph{intrinsic motivation} framework. There appears to be a growing scientific consensus in developmental psychology that human beings, from infants to adults, develop their understanding of the world using certain cognitive systems such as intuitive theories, social-structures, spatial systems, etc. ~\citep{Spelke2007,Lake2016}. During curiosity-driven, creative, or risk-taking activities, rational agents use this understanding to generate \textit{intrinsic goals}. Accomplishing these intrinsic goals leads to the accumulation of \textit{intrinsic rewards}, thereby exhibiting an innate desire to explore, manipulate, or probe their environment ~\citep{Oudeyer2007}.

Drawing parallels to reinforcement learning, the goal of a traditional RL agent is to maximize its expected cumulative reward. This behaviour is extrinsically motivated since the reward signal is external to an agent. We say that an agent is \textit{intrinsically motivated} if it has intrinsic goals and rewards. In the context of exploration for RL, the aim of the intrinsic motivation approach is to use intrinsic reward as a \textit{heuristic} that assigns an \textit{exploratory value} to the agent's actions. For example, an agent may receive intrinsic rewards for visiting novel parts of the environment that need further exploration ~\citep{Thrun1992a}.

Many formulations that quantify the exploratory value of an action has been put forth, and most of them augment the environment's reward function so as to motivate directed exploration. \cite{Schmidhuber2010} proposed a measure for intrinsic motivation by taking into account the improvement a learning algorithm effected on its predictive world model. This measure tracks the progress of an agent's ability to better compress the history of states and actions ~\citep{Steunebrink2013}. Another framework for intrinsically motivated learning is to maximize the \textit{mutual information}. An intrinsic reward measure called \textit{empowerment} is formulated by searching for the maximal mutual information ~\citep{Mohamed2015}. The notion of maximizing information gain was demonstrated in a humanoid robot by the introduction of \textit{artificial curiosity}~\citep{Schmidhuber1991} as an \textit{intrinsic goal} ~\citep{Frank2014}.

These formulations have some major drawbacks which hinder their suitability as exploration heuristics. Firstly, they fail to provide any strong theoretical guarantees of efficient exploration. \cite{Leike2016} pointed out that since none of these heuristics take into account the reward structure of the problem, they do not distinguish between regions of high and low expected reward. Secondly, these algorithms require that we maintain the environment dynamics of the underlying MDP, which prevents us from easily integrating them with model-free algorithms. Another major drawback is the computational overhead associated with calculating the heuristic. For problems with large state/action spaces, computing the intrinsic reward becomes  intractable for many heuristics ~\citep{Bellemare2016}. Most problems of interest have extremely large state spaces, and hence the intrinsic motivation heuristic is currently impractical as an exploration strategy in these domains.

\chapter{Exploration in Feature Space}
\label{cha:explore_phi}

\epigraph{\textit{`To wander is to be alive.'}}{Roman Payne, Europa}

In this chapter we introduce a simple, optimistic, count-based exploration strategy that achieves state-of-the-art results on high-dimensional RL benchmarks. In \cref{sec:explore_drawback} we begin by discussing the drawbacks of current exploration strategies. In \cref{sec:novelty_fa} we provide an exposition of the core ideas that underpin our algorithm. Finally, in \cref{sec:phi_eb_alg} we present our algorithm, as well as a number of related theoretical results.

\section{Drawbacks of Existing Exploration Methods for Large MDPs}
\label{sec:explore_drawback}
We introduced count-based exploration strategies for large MDPs in section \cref{sec:large_mdps}. Even though they are the current state-of-the-art exploration algorithms in these domains, we consider that there are some potential drawbacks to their common approach to estimating novelty. The motivation for our algorithm arises from trying to avoid these drawbacks.

\subsection{Choosing a Novelty Measure}
\label{sec:conf_novelty}
The aforementioned algorithms compute a generalized visit-count. This generalized count is a novelty measure that quantifies the (dis)similarity of a state to those in the history. These algorithms drive the agent towards regions of the state space with high novelty. However, the effectiveness of these novelty measures depends on the way in which they measure the similarity between states. If this similarity measure is not chosen in a principled way, states may deemed similar in ways that are not relevant to the given problem. Let us explore this issue by taking an example.
\begin{example}[Confounded novelty]
\label{eg:alice_exp}
     Alice is a foodie. She wants to explore the myriad restaurants that are open in her city. Suppose that Alice's novelty measure treats restaurants as similar if they are geographically close. Alice consults her novelty measure to choose a restaurant she has not tried yet, and it returns a Chinese restaurant in a distant suburb that she has not visited before. Alice scratches her head thinking: `I have been to a tonne of Chinese restaurants; if only my novelty measure understood that and suggested a different cuisine!' Unfortunately, her novelty measure considers this restaurant very dissimilar from the Chinese restaurants she has visited, simply because it is geographically distant from them.
\end{example}
The problem here is that Alice's novelty measure does not know anything about which features matter when evaluating the novelty of a restaurant. Let us now look at an example from the recent exploration literature where this problem can be clearly observed.

\subsubsection*{Inappropriate Novelty Measures in Practice}
The problems that can arise from an unprincipled choice of novelty measure are well illustrated in the experimental evaluation of~\cite{Stadie2015}. Their algorithm uses an autoencoder to encode the state-space into a lower dimensional representation. The encoding is then fed into a model dynamics prediction neural network which estimates the novelty by providing an error-based bonus. This method, called Model Prediction Exploration Bonuses (MP-EB), uses an error based estimator and is different from the visit-density model of \cite{Bellemare2016}, but they both estimate novelty. To generalize regarding value they use the DQN network, and so we refer to their algorithm as DQN+MP-EB.

\begin{figure}[H]
    \begin{minipage}{0.48\textwidth}
        \centering
        \includegraphics[scale=0.12]{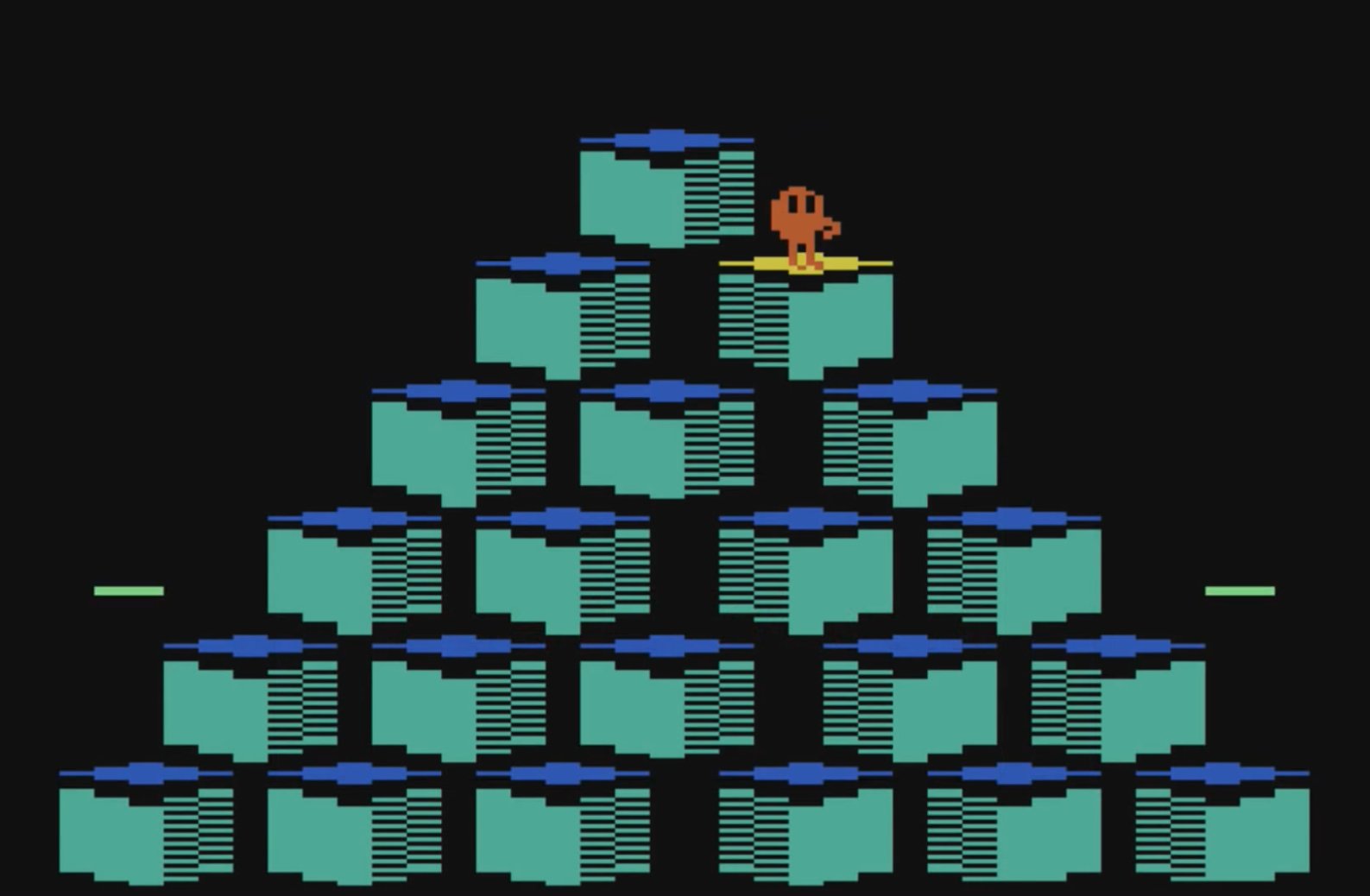}
        \caption{Q*bert Level 1}\label{fig:qbert_l1}
    \end{minipage}\hfill
    \begin {minipage}{0.48\textwidth}
        \centering
        \includegraphics[scale=0.12]{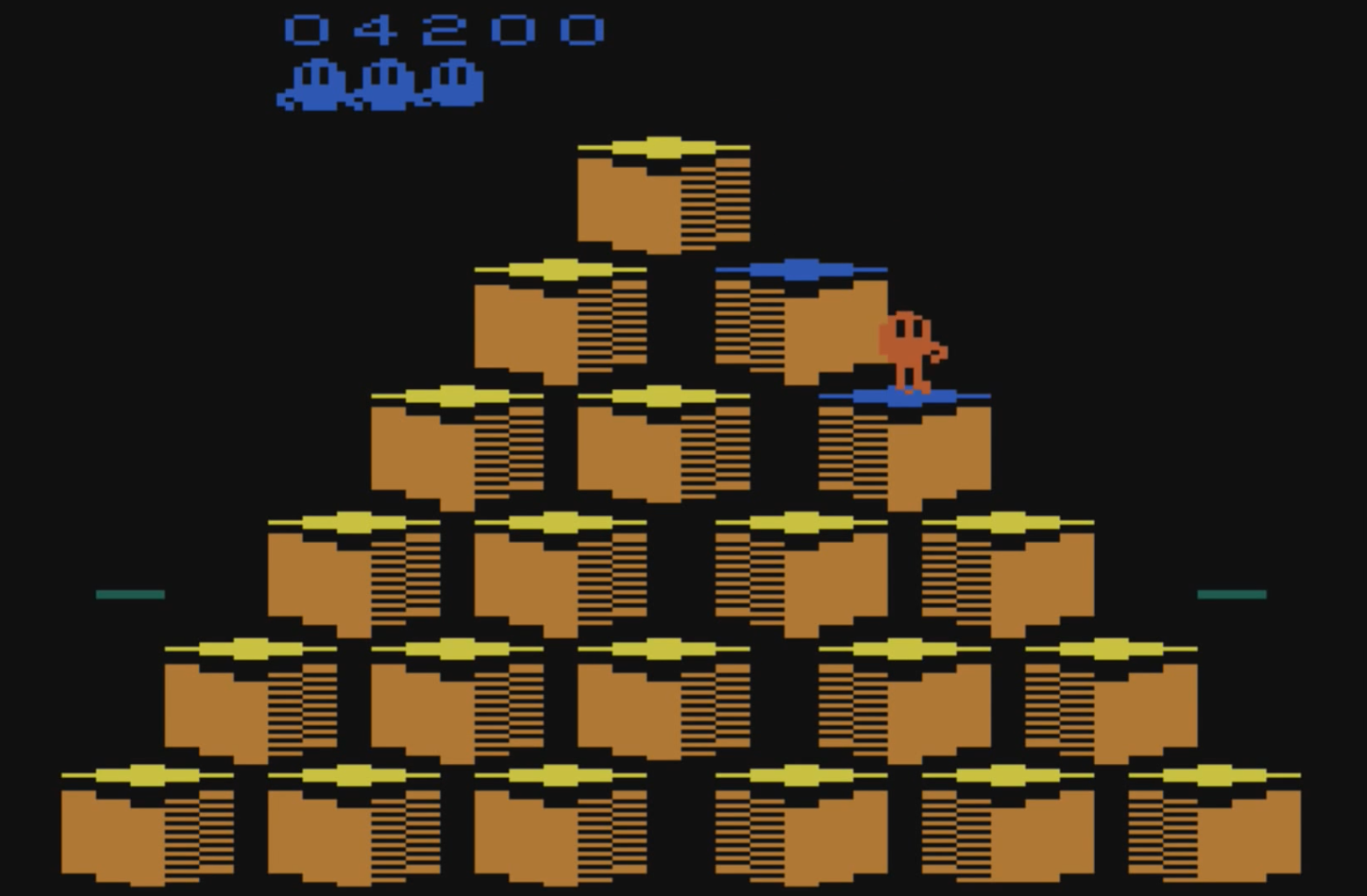}
        \caption{Q*bert Level 2}\label{fig:qbert_l2}
    \end{minipage}        
    \caption{Two levels of the Atari2600 game Q*bert}
    \label{fig:qbert_l12}
\end{figure}

During empirical evaluation of their algorithm an anomaly was detected in the game Q*bert\footnote{In Q*bert, the goal of the agent is to jump on all the cubes without falling off the edge, or being captured.} from the Arcade Learning Environment\footnote{ALE is a performance evaluation platform consisting of Atari2600 games. It is considered as the standard performance test bed for RL algorithms. We'll discuss in depth about ALE in \Cref{cha:methodology}.} (ALE) benchmarking suit. DQN+MP-EB algorithm scored lower than the baseline algorithm, DQN+$\epsilon$-greedy. They attributed this anomaly to the fact that during each level change of Q*bert, the color of the game changes dramatically, but neither the objective nor the structure of the level changes (\Cref{fig:qbert_l12}). When their agent reached level 2 (\Cref{fig:qbert_l2}), it perceived the state to be completely novel because MP-EB is sensitive to color. This tricked MP-EB into assigning high exploration bonus to all the states even though the action-values of the states hadn't changed. Hence the policy of the agent was impacted adversely.

The pathology of DQN+MP-EB in the Q*bert game highlights a serious problem with current novelty estimators \- they do not take into account the relevancy of a state to the task an agent is trying to accomplish. We argue that a measure of novelty should not just be an arbitrary generalized representation of how many times an agent has visited a state, but should ideally be a measure of dissimilarity in facets that are relevant to the agent's goal. Two states can be different in many ways; the challenge is to find out a similarity metric which is effective in achieving the agent's goal optimally. In Example \ref{eg:alice_exp}, Alice's novelty measure did not know that suggesting a restaurant with a different cuisine would be more relevant to her task, thereby naively suggesting a geographically distant unvisited restaurant.

\subsection{Separate Generalization Methods for Value and Uncertainty}
We contend that this deficiency is not peculiar to MP-EB, but rather that it may arise whenever the novelty measure is not designed to be task-relevant. Indeed, all of the aforementioned algorithms which compute a novelty measure share a common structure which leaves them vulnerable to this problem. Each algorithm has two quite unrelated components: a value estimator (an RL algorithm which performs policy evaluation), and a novelty estimator. Each component involves an entirely separate generalization method. The value estimator makes use of a feature representation of the state space in order to generalize about value. The novelty estimator separately utilizes a different, exploration-specific state space representation to measure the similarity between states. For example, the \#Exploration algorithm of \cite{Tang2016} uses the DQN algorithm for value estimation. In order to estimate novelty, however, \#Exploration maps the state space into a lower-dimensional representation using locality sensitive hashing. The similarity measure induced by the choice of hash codes is unlikely to resemble that which is induced by the features learnt by DQN. The DQN-CTS-EB algorithm of \cite{Bellemare2016} has a similar structure: DQN is used to estimate value, but the CTS density model is used to estimate novelty. Again, it is not obvious that there should be much in common between the two similarity measures induced by these different state space representations. One might think that this is natural; after all, each representation is used for a different purpose. However, there are two questions we can ask here. Firstly, is there redundant computation due to performing a dimensionality reduction of the same state-space twice? If so, can we reuse the same state space representation for both value and novelty estimation? We address these questions in the following section.

Before moving on we should note that the concerns we express in this section have already been raised in the literature. In their empirical evaluation \cite{Bellemare2016} observed that their value estimator (DQN) was learning at a much slower rate than their CTS density model (their novelty measure). The authors attribute this mismatch to the incompatibility between novelty and value estimators. They further go on to suggest that designing density models to be compatible with value function would be beneficial and a promising research direction.

The drawbacks we presented in this section suggest that there may be much room for improvement in the design of novelty estimators for exploration. In the following sections we describe our technique for estimating novelty by factoring in the insights we gained from analyzing these drawbacks. We first provide a solid footing for some of the assumptions that we made while designing the algorithm. We then go on to present our core exploration algorithm, and then combine it with a model-free RL algorithm (SARSA($\lambda$)). In the coming chapters we present empirical evidence that our RL algorithm achieves world-leading results on the ALE benchmarking suite.

\section{Estimating Novelty in Feature Space}
\label{sec:novelty_fa}

\subsection{Motivation}
Which representation of the state space is appropriate for novelty estimation? Intuitively, if we use some \textit{parameters} to determine the value of a \textit{state}, then naturally, two such objects are considered dissimilar only if they differ in these parameters. Analogously, if the agent is using certain features to determine the value of a state, then naturally, two such states should be considered dissimilar only if they differ in those value-relevant features. This motivates us to construct a similarity measure that exploits the feature representation that is used for value function approximation. These features are explicitly designed to be relevant for estimating value. If they were not, they would not permit a good approximation to the true value function. This sets our method apart from the approaches described in \cref{sec:large_mdps}, which measure novelty with respect to a separate, exploration-specific representation of the state space, one that bears no relation to the value function or the reward structure of the MDP. We argue that measuring novelty in feature space is a simpler and more principled approach, and hypothesise that more efficient exploration will result. Our proposal ensures that generalization regarding novelty is done in the same space as generalization regarding value. \cref{fig:blk_explore} illustrates the basic structure of our proposed novelty estimator.

\begin{figure}[H]
  \label{fig:blk_explore}
  \centering
  \includegraphics[scale=0.55]{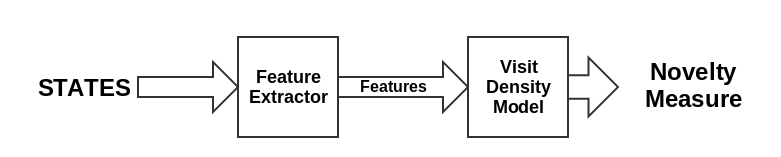}
  \caption{Novelty Measure in Feature Space}
\end{figure}

Let us make the idea more concrete with our running example.

\begin{example}[Value-relevant exploration]
After Alice's disappointing restaurant visit last time, she tweaked her novelty estimator such that it now generalizes based on value-relevant features like the type of cuisine, the star rating, and the other features that truly determine the quality of Alice's dining experience. When Alice is ready to try something new, she can rest assured that it's going to be something novel in a way that is meaningful.
\end{example}

\subsection{Design Decisions}
\label{sec:design_decisions}
Our exploration strategy, henceforth known as $\bm\phi$\textit{-exploration bonus} ($\bm\phi$-EB), can be thought of as exploration in the feature space. This makes the existence of a feature map crucial to our strategy. Therefore we require that our algorithm be compatible with Linear Function Approximation (LFA). Before the advent of neural networks and subsequently DQN, large RL problems used linear function approximation to estimate the value of a state. Our decision to use LFA as our value prediction module has the following desirable benefits:

\begin{itemize}        
    \item \textbf{Domain Independence: } The visit-density models that we have seen so far (MP-EB, CTS-EB, PixelCNN, etc.) are designed to work with RGB pixel values from a video input. Though there are many domains that use video input to train the agent, there are equally many other domains that have nothing to do with a video input. For example, reinforcement learning is used in the financial sector to optimize portfolios, asset allocations, and trading systems ~\citep{Moody2001}. Therefore developing a visit-density model that is domain independent is a key challenge. Our $\bm\phi$-EB method estimates the novelty using the same features that LFA uses to approximate the value function. This allows our exploration strategy to be compatible with any value-based RL algorithm that uses LFA.
    
    \item \textbf{Indirect dependence on LFA: } LFA is essentially a linear combination of features. The only requirement $\bm\phi$-EB has is the existence of a feature map, which is implicitly satisfied with LFA. Because of this indirect dependence on LFA, we hypothesis that it is possible to extend $\bm\phi$-EB to be compatible with value-networks that perform representation learning as well (\eg DQN). Due to resource and time constraints we do not pursue empirical evidence for this claim, but rather leave this as a possible future extension of our research.
    
    \item \textbf{Single point of change: } The best way to assess the performance impact of changes to a system is to confine the change to a single module and then run performance tests. Following this principle, we know that SARSA($\lambda$) is a value-based RL algorithm which uses LFA for value prediction and $\epsilon$-greedy for exploration~\citep{Sutton1998}. SARSA($\lambda$) has been studied, perfected and validated through-out the ages. Therefore showcasing the performance gains achieved by replacing $\epsilon$-greedy with our $\bm\phi$-EB exploration strategy allows for a sound empirical proof for the efficacy of our algorithm.
    
\end{itemize}
One drawback of using LFA for value prediction is that it requires a set of hand-crafted features. This is easily mitigated by choosing the Arcade Learning Environment (ALE) as our evaluation platform~\citep{Bellemare2013}, combined with the Blob-PROST feature set~\citep{Liang2015}. Using Blob-PROST as our feature set has and added advantage. Blob-PROST is designed to mimic the features learned by DQN, thus making our algorithm comparable with those using DQN for representation learning and value prediction. We'll discuss in depth about the ALE and Blob-PROST in~\cref{cha:methodology}.

\section{The \texorpdfstring{$\bm\phi$}{phi}-EB Algorithm}
\label{sec:phi_eb_alg}
The main original contribution of this work is a method for estimating novelty in feature space. The challenge is to do so without explicitly computing the distance between each new feature vector and all the feature vectors in the history. That approach quickly becomes infeasible because the cost of computing all these distances grows with the size of the history. Our method instead constructs a density model over feature space that assigns higher probability to states that share more features with more frequently observed states. In order to formally describe our method we first introduce some notation.
\subsubsection*{Notation}
\begin{itemize}
    \item $\bm\phi: \mathcal{S} \rightarrow \mathcal{T}\subseteq\mathbb{R}^M$, The feature map used in LFA. Maps the state space into an $M$-dimensional feature space, $\mathcal{T}$.
    \item $\bm\phi_t \equiv \bm\phi(s_t)$,
    Feature vector observed at time $t$, whose $i^{th}$ component is denoted by $\phi_{t,i}$
    \item $\bm\phi_{1:t}\equiv(\bm\phi_1,\ldots,\bm\phi_t)\in\mathcal{T}^t$, Sequence of feature vectors observed after $t$ timesteps.
    \item $\bm\phi_{1:t}\bm\phi\equiv(\bm\phi_1,\ldots,\bm\phi_t, \bm\phi)\in\mathcal{T}^{t+1}$, Sequence where $\bm\phi_{1:t}$ is followed by $\bm\phi$.
    \item $\mathcal{T}^*$, Set of all finite sequences of feature vectors.
    \item $\rho:\mathcal{T}^*\times\mathcal{T}\rightarrow[0,1]$, The sequential density model (SDM) that maps a finite sequence of feature vectors to a probability distribution.
\end{itemize}

We will now present the key component of our algorithm that allows us to estimate novelty in feature space.
\subsection{Feature Visit-Density}
\label{sec:feat_visit_density}
\begin{definition}[Feature visit-density]
The feature visit-density $\rho_t(\bm\phi)\equiv\rho(\bm\phi\,;\,\bm\phi_{1:t})$ at time $t$ is a probability distribution over the feature space $\mathcal{T}$, representing the probability of observing the feature vector $\bm\phi$ after observing the sequence $\bm\phi_{1:t}$. It is modelled as a product of independent factor distributions $\rho_t^i(\phi_i)$ over individual features $\phi_i$
\[\rho_t(\bm\phi) = \prod_{i=1}^M\rho_t^i(\phi_i)\]
\end{definition}

This density model induces a similarity measure on the feature space. Loosely speaking, feature vectors that share component features are deemed similar. This enables us to use $\rho_{t}(\bm{\phi})$ as a novelty measure for states, because it represents the frequency with which features are observed in the history. When confronted with a new state, we are able to estimate how frequently its component features have occurred in the history. If $\bm{\phi}(s)$ has more novel component features, $\rho_{t}(\bm{\phi})$ will be lower. By using a density model we are therefore able to measure novelty in a way that usefully generalizes the agent's uncertainty across the state space. To illustrate this, let us consider an example.

\begin{example}
Suppose we use a 3-D binary feature map and that after 3 timesteps the history of observed feature vectors is $\bm{\phi}_{1:3}=(0,1,0),(0,1,0),(0,1,0)$. Let us estimate the feature visit densities of two unobserved feature vectors $\bm{\phi'}=(1,1,0)$, and $\bm{\phi}''=(1,0,1)$. Using the KT estimator for the factor models, we have $\rho_{3}(\bm{\phi}')=\rho_{3}^{1}(1)\cdot\rho_{3}^{2}(1)\cdot\rho_{3}^{3}(0)=\frac{1}{8}\cdot\frac{7}{8}\cdot\frac{7}{8}\approx0.1$, and $\rho_{3}(\bm{\phi}'')=\rho_{3}^{1}(1)\cdot\rho_{3}^{2}(0)\cdot\rho_{3}^{3}(1)=(\frac{1}{8})^{3}\approx0.002$. Note that $\rho_{3}(\bm{\phi}')>\rho_{3}(\bm{\phi}'')$ because the component features of $\bm{\phi}'$ are more similar to those in the history. As desired, our novelty measure generalizes across the state space.
\end{example}

Each factor distribution $\rho_t^i(\phi_i)$ is modelled using a count-based estimator. A naive option would be to use the empirical estimator which is the ratio of the number of times a feature has occurred to the total number of time steps. Another class of count-based estimators are the Dirichlet estimators which enjoy strong theoretical guarantees~\citep{Hutter2013}. We use the Krichevsky-Trofimov(KT) estimator which is a Dirichlet-like estimator that is simple, easy to implement, scalable, and data efficient \citep{Krichevsky1981}. If $N_t(\phi_i)$ is the number of times the feature $\phi_i$ has been observed, then the KT estimator is given by:
\[\rho_t^i(\phi_i) = \frac{N_t(\phi_i)+\frac{1}{2}}{t+1}\]

Using independent factor distributions for modelling the probability of each feature component inherently assumes that the features are independently distributed. This is not always the case, especially in video-input based domains such as the ALE we have many features that are strongly correlated. This doesn't mean that we cannot use fully factorized distributions. One of the early assumptions made by ~\cite{Bellemare2016} about the density model is that the states are independently distributed. This allowed them to factorize the states, and model each factor using a position-dependent CTS\footnote{A Bayesian variable-order Markov model.} density model. Moreover, our empirical evaluations show that we achieve world leading results in hard exploration games suggesting that independent factored distributions produce good novelty measures. Thus by precedence and by empirical data the independence assumption on the features is a well-justified trade-off that makes the computation of novelty fast and data efficient.

\subsection{\texorpdfstring{The $\bm\phi$}{phi}-pseudocount}
Here we adopt a recently proposed method for computing generalised visit-counts from density models \citep{Bellemare2016}. By analogy with the pseudocounts presented in that work, we derive two $\phi$-pseudocounts from our feature visit-density. Both variants presented generalize the same quantity, the state visitation count function $N_t(s)$. The expression given in the following definition is derived in \cite{Bellemare2016}. We emphasize that our approach constitutes a departure from theirs, because while they derive pseudocounts from a \textit{state} visit-density model, we do so using a \textit{feature} visit-density model.
\begin{definition}[$\bm\phi$-pseudocount]
Let $\rho_t'(\bm\phi)\equiv\rho_t(\bm\phi\,;\,\bm\phi_{1:t}\bm\phi)$\footnote{Also called the \textit{recoding probability}.} be the probability that the feature visit-density model would assign $\bm\phi$ if it was observed one more time. Then the $\bm\phi$-pseudocount for a state $s\in\mathcal{S}$ is given by:
\[\hat{N}_t^\phi(s) = \frac{\rho_t(\bm\phi(s))(1-\rho'_t(\bm\phi_t(s)))}{\rho_t'(\bm\phi(s))-\rho_t(\bm\phi(s))}\]
\end{definition}

\subsection{The \texorpdfstring{$\bm\phi$}{phi}-Exploration Bonus algorithm \texorpdfstring{$(\bm\phi$}{(phi}-EB\texorpdfstring{$)$}{)}}

Equipped with all the tools necessary for the construction of an exploration bonus we now proceed to define the $\bm\phi$-EB algorithm. We provide a high level flow-chart for the construction of the bonus in \cref{fig:phi_eb_flow}, and the corresponding pseudo-code in Algorithm \ref{alg:phi_eb_calc}. Having defined the $\phi$-pseudocount (a generalised visit-count), we follow traditional count-based exploration algorithms by computing an exploration bonus that depends on this count. The functional form of the bonus is the same as in MBIE-EB\@; we merely replace the empirical state-visit count with our $\phi$-pseudocount.

\begin{definition}[$\bm\phi$-exploration bonus]
\label{def:phi_exp_bonus}
The exploration bonus for a state-action pair $(s,a)\in \mathcal{S}\times\mathcal{A}$ at time $t$ is
\[\mathcal{R}_t^\phi(s,a) = \frac{\beta}{\sqrt{\hat{N}_t^\phi(s)}}\]
where $\beta$ is a hyper-parameter that controls the agents level of optimism.
\end{definition}

Loosely speaking, the hyper-parameter $\beta$ can be viewed as a knob that tunes the agent's confidence in its estimate of the true action-value function. Higher values of $\beta$ makes the agent under-confident about value, and result in too much exploration. Very low $\beta$ values do not encourage enough exploration because the exploration bonus is too small to dissuade the agent from acting greedily with respect to its current value estimates. In both scenarios the final policy of the agent is affected adversely. The goal is to find a $\beta$ value that gives good results across domains. We performed a coarse parameter sweep among the games in the ALE evaluation platform and concluded that $\beta=0.05$ was the best value. Further details regarding the selection of $\beta$ value is discussed in \cref{sec:beta_sweep}.

\begin{figure}[H]
  \label{fig:phi_eb_flow}
  \centering
  \includegraphics[scale=0.51]{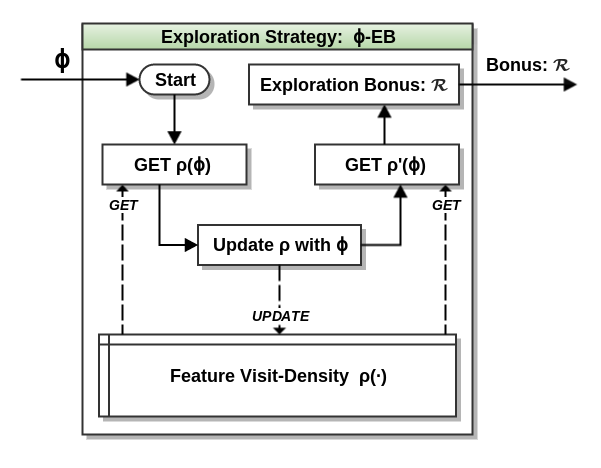}
  \caption{Flow Chart for computing the exploration bonus of $\bm\phi$-EB}
\end{figure}

\begin{algorithm}[H]
\caption{ $\bm\phi$-exploration bonus}
\label{alg:phi_eb_calc}
\SetAlgoLined
\DontPrintSemicolon
\KwIn{Density Model $\rho$}
\SetKwProg{Fn}{function}{}{end function}
\Fn{\textsc{FeatureVisitDensity}{$(\bm\phi)$}}{
\Return $\prod\limits_{i=1}^M\rho^i(\phi_i)$\;
}
\BlankLine
\BlankLine
\KwIn{Feature Visit Count $N_t$; Density Model $\rho$; Current Timestep $t$}
\setcounter{AlgoLine}{0}
\Fn{\textsc{UpdateFeatureVisitDensity}{$(\bm\phi)$}}{
    \For{i=1 \emph{\KwTo} $M$}{
        $\rho^i(\phi_i) \leftarrow \dfrac{N_t(\phi_i)+\frac{1}{2}}{t+1}$\;
    }
}
\BlankLine
\BlankLine
\setcounter{AlgoLine}{0}
\SetKwProg{Fn}{function}{}{end function}
\Fn{\textsc{PseudoCount}{$(p,p')$}}{
\Return $\dfrac{p(1-p)}{p'-p}$\;
}
\BlankLine
\BlankLine
\setcounter{AlgoLine}{0}
\KwIn{LFA Feature Map $\bm\phi$; Exploration Coefficient $\beta$}
\SetKwProg{Fn}{function}{}{end function}
\Fn{\textsc{ExplorationBonus}{$(s)$}}{
$\rho(\bm\phi)$ $\leftarrow$ \textsc{FeatureVisitDensity}{$(\bm\phi(s))$}\;
\textsc{UpdateFeatureVisitDensity}{$(\bm\phi(s))$}\;
$\rho'(\bm\phi)$ $\leftarrow$ \textsc{FeatureVisitDensity}{$(\bm\phi(s))$}\;
$\hat{N}^\phi(s) \leftarrow $ \textsc{PseudoCount}{$(\rho(\bm\phi), \rho'(\bm\phi))$}\;
\Return $\dfrac{\beta}{\sqrt{\hat{N}^\phi(s)}}$\;
}
\end{algorithm}

\subsection{LFA with \texorpdfstring{$\bm\phi$}{phi}-EB}
One the advantages that we have in developing our algorithm for use with LFA is that our exploration strategy is compatible with all value-based RL algorithms that use LFA. As we will see, our empirical performance across a range of environments suggests that one can plug our exploration strategy with little to no modification into any of these algorithms and expect considerable gains in exploration efficiency. In our empirical evaluation we use SARSA($\lambda$) with replacing traces as our value-based reinforcement learning algorithm. Algorithm \ref{alg:rl_with_eb} presents the pseudo-code for a generic RL algorithm that uses the augmented reward $r^+$ for updating the function parameters $\bm\theta$ of the approximate action-value function $\hat Q(s, a) = \bm{\theta}^T \bm{\phi}(s, a)$.

\begin{algorithm}[H]
\caption{LFA with $\bm\phi$-EB}
\label{alg:rl_with_eb}
\SetAlgoLined
\DontPrintSemicolon
\KwIn{LFA Feature Map $\bm\phi$; Training Horizon $t_{end}$}
$t\leftarrow0$\;
Initialize arbitrary $\bm\theta_t$\;
$s_t,a_t\leftarrow$ initial state, action\;
\While{$t<t_{end}$}{
$r_{t+1},s_{t+1} \leftarrow$ \textsc{Act}{$(a_t)$}\;
$\mathcal{R}_t^\phi(s_t,a_t)\leftarrow$ \textsc{ExplorationBonus}{$(s_t)$}\;
$r_{t+1}^+ \leftarrow r_{t+1} + \mathcal{R}_t^\phi(s_t,a_t)$\;
$a_{t+1} \leftarrow$ \textsc{NextAction}{$(s_{t+1}, \bm\theta_{t})$}\;
$\bm\theta_{t+1}\leftarrow$\textsc{UpdateTheta}{$(r_{t+1}^+, \bm\phi)$}\;
$t\leftarrow t+1$\;
}
\Return $\bm\theta_{t_{end}}$\;
\footnotetext{The functions \textsc{NextAction} and \textsc{UpdateTheta} are specific to the underlying value-based RL algorithm used, hence left unspecified. \textsc{Act}{$(a_t)$} performs action $a_t$ in the environment.}
\end{algorithm}

\subsection{Complexity Analysis}
\subsubsection*{Time Complexity}
From Algorithm~\ref{alg:phi_eb_calc} it is trivial to see that a call to \textsc{ExplorationBonus} has a worst-case time complexity of $O(M)$, where $M$ is the dimension of the feature space. This suggests that the time needed to compute the novelty of a state is independent of the dimension of the state-space. Also, more often than not, the dimension of the feature space is far smaller than that of the state space. Therefore, our algorithms generates significant savings in computation over other density models whose time-complexity scales with the number states. In practice, for a binary feature set like Blob-PROST we process only those features that have been observed before. This is achieved by maintaining a single prototypical factor density estimator for all previously unseen features. We'll discuss the implementation specific details in depth in \cref{cha:methodology}.

\subsubsection*{Space Complexity}
We look at Algorithm~\ref{alg:rl_with_eb} to analyze what objects are needed to be persisted across iterations so as to facilitate calculation of the exploration bonus. Clearly the factor density estimators $\rho^i(\phi_i)$, and the feature visit count $N_t(\phi_i)$ are needed to evaluate and update the feature visit density. Therefore it can be seen that our algorithm has a worst case space complexity of $O(M)$. Again, because the features in Blob-PROST are binary valued, the KT estimator can be defined recursively. This allows for updating the factor density online without the need to maintain a feature visit count $N_t(\phi_i)$. We'll discuss more on this in \cref{cha:methodology}.

\section{Summary}
In this chapter we have presented the main contribution of our research. Motivated by the drawbacks of current state-of-the-art exploration algorithms, we introduced our novel exploration algorithm called $\bm\phi$-EB. Later, we provided an exposition on the various components of the algorithm and also analysed its time and space complexity.

Now that we have presented our algorithm, we move on to implementation aspects. The next chapter provides a detailed overview of the evaluation test-bed, the software architecture, and the implementation challenges faced during the Research \& Development of the algorithm.

\chapter{Implementation}
\label{cha:methodology}

\epigraph{\textit{'Any A.I. smart enough to pass a Turing test is smart enough to know to fail it.'}}{Ian McDonald, \textit{River of Gods}}

This chapter is dedicated to developing a technically correct implementation of our exploration strategy $\bm\phi$-EB, and its surrounding infrastructure. This allows us to perform a sound empirical evaluation which is the focus of the next chapter.

In \cref{sec:sw_arch} we present the high-level architecture of the whole system, and how the various components interact with each other. Later, in \cref{sec:implementation}, we present the implementation of our exploration strategy, $\bm\phi$-EB. Throughout the section we also talk about the design aspects, and optimization's that went into implementing $\bm\phi$-EB.

\section{Software Architecture}
\label{sec:sw_arch}

Our implementation goal is to develop an RL software agent that uses $\bm\phi$-EB as its exploration strategy. We present the high-level design of the algorithm in \cref{fig:soft_arch}. The presented diagram is analogous to the Agent-Environment interaction cycle (\cref{fig:agent_env}), but with more granularity. From an exploration-centric standpoint, we first provide a concise overview of the components presented in the architecture, and then an exposition on the implementation details for $\bm\phi$-EB.

\begin{figure}[H]
  \label{fig:soft_arch}
  \centering
  \includegraphics[scale=0.52]{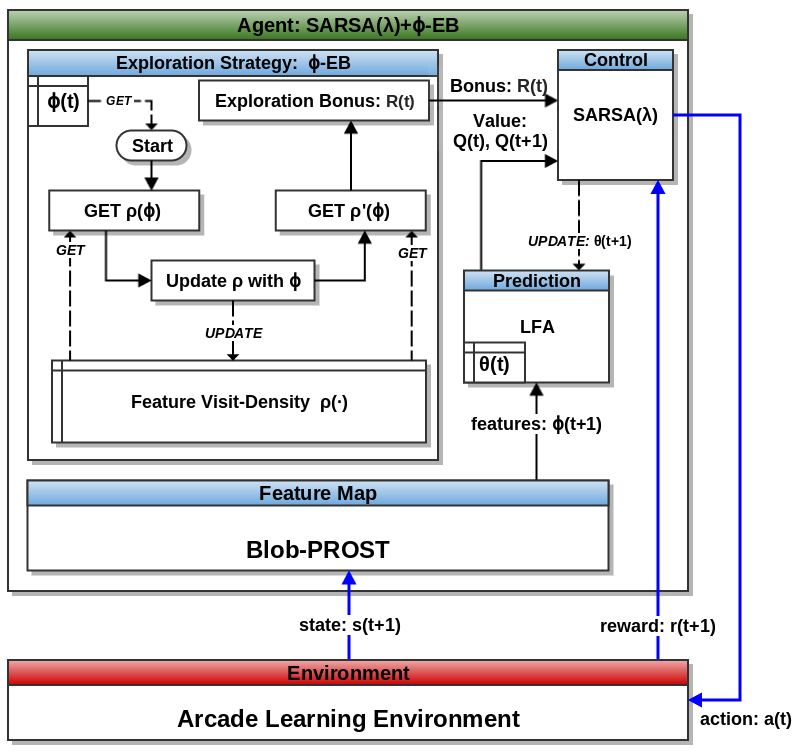}
  \caption[Agent-Environment interaction framework for SARSA$(\lambda)$+$\bm\phi$-EB]{Agent-Environment interaction framework for SARSA$(\lambda)$+$\bm\phi$-EB.\footnotemark}  
\end{figure}
\footnotetext{Boxes with a tiny row and column, on top and left edges respectively, denote objects stored in RAM. They can persists across cycles and episodes. Dotted arrows with instruction on them denote operation on such objects.}

\subsection{Modular Overview}

\subsubsection*{Control}
We use SARSA$(\lambda)$ with replacing traces \citep{Sutton1998} as our learning algorithm. This decision was driven primarily by two factors. First, using the Blob-PROST\footnote{Discussed in \cref{sec:feature_set}.} feature set meant that we are locked into the framework provided by \cite{Liang2015}. In our case this is in fact desirable. Replacing the exploration module of an open source, peer-reviewed and published implementation with our own exploration module enhances the credibility of any performance gains that result. Second, we need a learning algorithm that works well with Linear Function Approximation (LFA).  When coupled with LFA, SARSA$(\lambda)$ has better convergence guarantees than Q-learning~\citep{Melo2008}. Hence, SARSA$(\lambda)$ is a suitable value estimation algorithm for our agent.

\subsubsection*{Exploration Strategy}
This component is our $\bm\phi$-EB exploration strategy that was proposed in \cref{cha:explore_phi}. We implement it using the \texttt{C++11}\footnote{\texttt{C++11} is a major revision of \texttt{C++}. This particular version was chosen because it makes several useful additions to core language libraries.} programming language. \texttt{C++} offers a significant edge over other languages in terms of efficiency and greater control over memory management. Due to the high dimensional nature of our problem, we need to extract as much performance as possible from our code. Therefore implementing the exploration strategy in \texttt{C++} is critical to the empirical success of our algorithm. Moreover, a lock-in with the framework provided by \cite{Liang2015} meant that there was no compelling reason to choose a different programming language. In the coming sections we provide a detailed look at the design and implementation of $\bm\phi$-EB.

\subsubsection*{Prediction}
We use LFA to generalize the action-value function for unknown state-action pairs. For further discussion on LFA we refer the reader to \cref{sec:fa}. LFA uses the Blob-PROST feature set from \cite{Liang2015} to approximate the action-value function.

\subsubsection*{Feature Map}
We consider the feature map to be an integral part of the agent. The ability to discern different features of a state is imperative to generalization regarding value, and by extension, to exploration. Our agent explores in the feature space, and we want to use LFA for value prediction. This necessitated the need for an efficient and effective feature set. The Blob-PROST feature set from \cite{Liang2015} is the best feature set available to date for the Arcade Learning Environment (ALE) evaluation platform. More details on Blob-PROST in \cref{sec:feature_set}.

\subsubsection*{Environment}
We chose the Arcade Learning Environment as our evaluation platform for the following reasons.
\begin{itemize}
    \item ALE contains many games which vary in degree of exploration hardness. This allows us to test the efficacy of our algorithm on a broad spectrum of games \citep{Bellemare2016}.
    \item ALE is widely accepted as the standard for testing RL algorithms. The vast majority of exploration specific research that is published post \cite{Bellemare2013} has adopted the ALE platform to report empirical results \citep{Mnih2013,Mnih2015,Stadie2015,Osband2016a,Bellemare2016,Tang2016,Ostrovski2017}. Therefore, in order to compare and contrast our results with existing research, it is crucial that we choose ALE as our evaluation platform.
\end{itemize}
More details on the Arcade Learning Environment in \cref{sec:ale}.


\subsection{Agent-Environment Work-flow}

We want to seamlessly integrate $\bm\phi$-EB into the control module. Therefore understanding the nuances of what happens in an agents cycle from an implementation perspective is critical. \cref{fig:soft_arch} also doubles as a work-flow diagram for our agent \agent. Following the usual agent-environment interaction process at timestep $t$, the agent performs an action $a_t$ on the environment and receives an extrinsic reward $r_{t+1}$. The agent also observes the new state of the environment, $s_{t+1}$. Inside the agent, the Blob-PROST feature map consumes the current state $s_{t+1}$ and returns a feature vector $\bm\phi_{t+1}$. The feature vector is then used by the LFA module to do value prediction. The $\bm\phi$-EB module uses the stored feature $\bm\phi_t$ to generate the exploration bonus $\mathcal{R}_t(s_t, a_t)$\footnote{Here we can see that generalization regarding value and novelty are being done in the same space.}. SARSA$(\lambda)$ updates the parameters of LFA with the TD update and chooses the next action optimistically.

\subsubsection*{$\bm\phi$-EB}
In the exploration strategy module the feature visit-density $\rho$ is a hash map that persists in memory across cycles and episodes. Each entry of $\rho$ is a key-value pair mapping individual features $\phi_i$ to its corresponding factor distribution $\rho_i$. In the $\bm\phi$-EB module shown in \cref{fig:soft_arch}, the flow of control is as follows: compute $\rho(\bm\phi)$ as product of factors, update $\rho$ with the observation $\bm\phi$, then compute $\rho(\bm\phi)$ again. Now calculate the pseudo-count and subsequently the exploration bonus $\mathcal{R}_t(s_t)$. The bonus $\mathcal{R}_t(s_t)$ is considered as an intrinsic reward and is sent to the control module, SARSA$(\lambda)$.

\subsubsection*{LFA}
The prediction module (LFA) approximates the next state action-value function using the parameter vector $\bm\theta_t$ as $\hat Q^\pi(s_{t+1}, a_{t+1}) = \bm{\theta_t}^T \bm{\phi}(s_{t+1}, a_{t+1})$. LFA sends the next state $Q$-value to the control module, SARSA$(\lambda)$. The parameter vector $\bm\theta_t$ is also an object that is saved in memory and persisted across cycles and episodes.

\subsubsection*{SARSA$(\lambda)$\footnote{For brevity we have left out the discussion on eligibility traces.}}
All the results from the various modules flow into the control module SARSA$(\lambda)$. The control module essentially has two tasks.
\begin{itemize}
    \item \textbf{Choose the next action $a_{t+1}$}\\
    The next state is chosen by being greedy with respect to next state action-value obtained from LFA.
    \[a_{t+1} = \arg\max_{a\in\mathcal{A}}\Big[\hat Q^\pi(s_{t+1},a)\Big]\]
    \item \textbf{Update $\bm\theta$ of LFA}\\
    First we augment the extrinsic reward $r_{t+1}$ with the intrinsic reward $\mathcal{R}_t(s_t)$ obtained from $\bm\phi$-EB module.
    \[r_{t+1}^+ = r_{t+1} + \mathcal{R}_t(s_t)\]
    The augmented reward $r_{t+1}^+$, the next state action-value $\hat Q^\pi(s_{t+1}, a_{t+1})$, and the current state action-value $\hat Q^\pi(s_{t}, a_{t})$, both from LFA, is used to calculate the TD error.
    \[\delta_{t+1} = r_{t+1}^+ + \gamma \hat Q^\pi(s_{t+1}, a_{t+1}) - \hat Q^\pi(s_t, a_t)\]
    Where $\gamma$ is the discount factor. Next we update $\bm\theta$, and is updated using the usual TD update formula.
    \[\bm\theta_{t+1} \leftarrow \bm\theta_t + \alpha\delta_{t+1}\]
    Where $\alpha$ is the learning rate.
\end{itemize}

Now that we have a clear idea about the surrounding infrastructure, let's move on to the implementation details of $\bm\phi$-EB

\section{Implementation Details}
\label{sec:implementation}

\subsection{Feature Visit-Density}
\label{sec:fvd_impl_det}
The central data structure that stores the factor distribution of each individual feature is an \texttt{unordered\_map}\footnote{Essentially a hash map } called \texttt{fvd\_map}$\langle\phi_i, \rho^i\rangle$. Each entry is a key-value pair mapping individual features to its corresponding factor distribution. This allows us to have constant time look-up for the factor distribution of any feature.
At first glance of the theoretical formulation of feature visit-density (\Cref{sec:feat_visit_density}, Algorithm \ref{alg:phi_eb_calc}), the implementation looks straight forward. Unfortunately that is not the case. We need to take into account certain implementation specific aspects that are often subsumed by mathematical formulation. Following are some of the important implementation details that need to be considered for computing feature visit-density. 
\begin{itemize}
    \item \textbf{Sparse Feature Vector}\\
    In practice the feature vector $\bm\phi$ is the list of features that are active in the current timestep. Most of the time the set of observed features is in a vastly smaller subspace of the feature space $\mathbb{R}^M$. Therefore, iterating till $M$ to compute the product of the factor distributions is quite wasteful. In order to overcome this we maintain a \textit{prototype}\footnote{In this context, a prototype function creates an object of a specified type. Here, a KT estimator which has seen $t$ zeros.} function that computes the KT-estimate of observing the feature give that it has never been observed in $t$ timesteps. Now whenever a new feature is observed it is added to \texttt{fvd\_map} with the current value of the prototype. If $M_t$ is the total number of features observed till timestep $t$, then we can compute the feature visit-density in $O(M_t)$ time.
    
    \item \textbf{Numerical Stability}\\
    Experience has taught us that when dealing with probabilities, innocent looking formulas such as ours can be deceiving. Since we are taking product of probabilities, they are bound to numerically underflow. In our implementation, rather than computing $\prod\limits_{i=1}^M\rho^i(\phi_i)$ we compute $\sum\limits_{i=1}^M\log\big(\rho^i(\phi_i)\big)$. This allows us to safely perform probability calculations without the worry of underflow.
    
    \item \textbf{Inactive Features}\\
    During the evaluation of the feature visit-density we need to consider the factor distributions for the features that are inactive but previously observed. Since we have already observed $\bm\phi_t$, we can identify the features in \texttt{fvd\_map} that are not active. The probability density stored in \texttt{fvd\_map} against some feature $\phi_i$, is the probability of $\phi_i$ being active. Assuming $\phi_i\not\in\bm\phi_t$, the probability of $\phi_i$ not being active is given by $\Big(1-$\texttt{fvd\_map$[\phi_i]\Big)$}. Therefore when evaluating feature visit-density for $\bm\phi_t$ we should also factor in the probability of inactive features not occurring.
\end{itemize}

Algorithm \ref{alg:impl_visit_dens} presents the implementation for computing the feature visit-density with all the above mentioned optimization/requirements. One key observation is that we return the log-probability. This is done to facilitate further log based probability computation that occur in other modules.

\begin{algorithm}[H]
\caption{Implementation of Feature Visit Density}
\label{alg:impl_visit_dens}
\SetAlgoLined
\DontPrintSemicolon
\SetKwProg{Fn}{function}{}{end function}
\KwIn{Current Timestep $t$}
\Fn{\textsc{KT$\_$Prototype}{$()$}} {
    \Return $\dfrac{0.5}{t + 1}$\;
}
\BlankLine
\BlankLine
\setcounter{AlgoLine}{0}
\KwIn{Factor Distribution Map \texttt{fvd\_map}$\langle\phi_i, \rho^i\rangle$}
\Fn{\textsc{LogFeatureVisitDensity}{$(\bm\phi)$}} {
    $sum\_log\_rho \leftarrow 0$\;
    \For{$i=1$ \emph{\KwTo} $|\bm\phi|$} {
        \If(\tcp*[f]{$O(1)$ look-up}){$\phi_i \not\in$ \texttt{fvd\_map.keys}} {
            \texttt{fvd\_map}$[\phi_i]$ = \textsc{KT$\_$Prototype}{$()$}\;
        }
        $sum\_log\_rho \leftarrow sum\_log\_rho + \log\big(\texttt{fvd\_map}[\phi_i]\big)$\;        
    }
    \tcc*[l]{Inactive features}
    \For{$i=1$ \emph{\KwTo} \texttt{size(fvd\_map.keys)}} {
        \If(\tcp*[f]{$O(1)$ look-up with flag trick}){$\phi_i \not\in$ $\bm\phi$} {
        $sum\_log\_rho \leftarrow sum\_log\_rho + \log\big(1 - \texttt{fvd\_map}[\phi_i]\big)$\;        
        }
    }
    \Return $sum\_log\_rho$\;
}
\end{algorithm}

\subsection{Updating Factor Densities}

Recall that we use the Krichevsky-Trofimov (KT) estimator to compute the factor densities. Given a sequence of symbols, the KT-estimator computes the probability of the next symbol. For a binary symbol-set, the KT-estimator is given by.
\[Pr(x_{t+1}=1\mid x_{1:t}) = \frac{n_1 + \frac{1}{2}}{n_0 + n_1 + 1}\]
Where $n_1$ is the number of 1's seen so far in the sequence, and $n_0$ is the number of 0's seen so far.

The Blob-PROST feature set is binary valued, making the use of KT-estimators ideal. Therefore, our factor density for a feature $\phi_i$ being active is given by.
\[\rho^i(\phi_i)\equiv Pr(\phi_i=1\mid \phi_{1:t}^i)=\frac{N_t(\phi_i) + \frac{1}{2}}{t + 1}\]
And the probability for the feature being inactive is.
\[\rho^i(\phi_i=0) = 1- \rho^i(\phi_i=1)\]

Where $N_t(\phi_i)$ is the number of times feature $\phi_i$ has been seen, and $\phi_{1:t}^i$ is the complete sequence of past observations for feature $\phi_i$.

The factor density equation is neat and simple, but it requires that we maintain a count for each feature. This is an unnecessary overhead and we can do better. We now propose an update formula for $\rho^i(\phi_i)$ and derive it.
\begin{proposition}[Update formula for KT-estimate $\rho^i(\phi_i)$]
\label{prop:rho_update}
The factor distribution $\rho_t^i$ at timestep $t$ for feature $\phi_i$ can be updated using the following update formula.
\[\rho_{t+1}^i(\phi_i) = \rho_{t}^i(\phi_i)\Bigg(\frac{t+1}{t+2}\Bigg) + \frac{\phi_i}{t+2}\]
Where $\phi_i\in\{0,1\}$
\end{proposition}
\begin{proof}
From the equation for KT-estimates of $\rho_{t}^i(\phi_i)$ we have,
\begin{align*}
    \rho_{t}^i(\phi_i) = \frac{N_t(\phi_i) + \frac{1}{2}}{t + 1}\\
    \rho_{t}^i(\phi_i)\Bigg(\frac{t+1}{t+2}\Bigg) = \frac{N_t(\phi_i) + \frac{1}{2}}{t + 2}\tag{1}
\end{align*}
In the next timestep $t+1$, depending on the value of $\phi_i$ we have two cases.
\begin{itemize}
    \item \textbf{Case 1:} Feature $\phi_i$ is active, \ie $\phi_i=1$\\
    The KT-estimate $\rho_{t+1}^i(\phi_i)$ can be written as,
    \begin{align*}
        \rho_{t+1}^i(\phi_i) &= \frac{N_{t+1}(\phi_i) + \frac{1}{2}}{(t+1) + 1}\\
        &= \frac{N_{t}(\phi_i) + 1 + \frac{1}{2}}{t + 2} \tag{Since $\phi_i=1$}\\
        &= \frac{\Big(N_{t}(\phi_i) + \frac{1}{2}\Big) + 1}{t + 2}\\
        \rho_{t+1}^i(\phi_i) &= \frac{N_{t}(\phi_i) + \frac{1}{2}}{t+2} + \frac{1}{t + 2} \tag{2}\\
    \end{align*}
    
    \item \textbf{Case 2:} $\;\;\phi_i=0$\\
    The KT-estimate $\rho_{t+1}^i(\phi_i)$ can be written as,
    \begin{align*}
        \rho_{t+1}^i(\phi_i) &= \frac{N_{t+1}(\phi_i) + \frac{1}{2}}{(t+1) + 1}\\
        &= \frac{N_{t}(\phi_i) + 0 + \frac{1}{2}}{t + 2} \tag{Since $\phi_i=0$}\\
        &= \frac{\Big(N_{t}(\phi_i) + \frac{1}{2}\Big) + 0}{t + 2}\\
        \rho_{t+1}^i(\phi_i) &= \frac{N_{t}(\phi_i) + \frac{1}{2}}{t+2} + \frac{0}{t + 2} \tag{3}\\
    \end{align*}    
\end{itemize}

In both cases, from Eq. $(2)$ and $(3)$ we can see that the value $\phi_i$ decides the existence of an additional term. Therefore by observation we can combine the two cases as follows.
\[\rho_{t+1}^i(\phi_i) = \frac{N_{t}(\phi_i) + \frac{1}{2}}{t+2} + \frac{\phi_i}{t + 2} \tag{4}\]
Therefore, from Eq. $(1)$ and $(4)$ we get,
\[\rho_{t+1}^i(\phi_i) = \rho_{t}^i(\phi_i)\Bigg(\frac{t+1}{t+2}\Bigg) + \frac{\phi_i}{t+2}\]
\end{proof}

Algorithm \ref{alg:impl_update_vd} presents the algorithm for updating the factor distributions. It uses the update formula presented in Proposition \ref{prop:rho_update} to efficiently update the factor distributions. In the implementation we can see that the update is performed in a two part manner with linear time complexity, rather than a naive double-loop search.

\begin{algorithm}[H]
\caption{Factor Distribution Update}
\label{alg:impl_update_vd}
\SetAlgoLined
\DontPrintSemicolon
\SetKwProg{Fn}{function}{}{end function}
\KwIn{Factor Distribution Map \texttt{fvd\_map}$\langle\phi_i, \rho^i\rangle$; Current Timestep $t$}
\Fn{\textsc{Update}{$(\bm\phi)$}} {
    \For{$i=1$ \emph{\KwTo} \texttt{size(fvd\_map.keys)}} {
        $\texttt{fvd\_map}[\phi_i] =  \texttt{fvd\_map}[\phi_i]\cdot\Bigg(\frac{t+1}{t+2}\Bigg)$\;
    }
    \For{$i=1$ \emph{\KwTo} $|\bm\phi|$} {
        $\texttt{fvd\_map}[\phi_i] = \texttt{fvd\_map}[\phi_i] + \frac{1}{t+2}$
    }
}
\end{algorithm}

\subsection{Exploration Bonus}

This is the entry point for our exploration strategy $\bm\phi$-EB. Due to the modular design of our algorithm, this function mostly acts like a hub that calls other functions sequentially to get the data required to calculate the exploration bonus. Algorithm \ref{alg:impl_exp_bonus} presents the implementation to calculate exploration bonus. Note that the probabilities are in log space to avoid numerical stability issues.

\begin{algorithm}[H]
\caption{Exploration Bonus}
\label{alg:impl_exp_bonus}
\SetAlgoLined
\DontPrintSemicolon
\SetKwProg{Fn}{function}{}{end function}
\KwIn{Exploration Coefficient $\beta$}
\Fn{\textsc{ExplorationBonus}{$(\bm\phi)$}} {
    $\log(\rho) \leftarrow$\textsc{LogFeatureVisitDensity}{$(\bm\phi)$}\;
    \textsc{Update}{$(\bm\phi)$}\;
    $\log(\rho') \leftarrow$\textsc{LogFeatureVisitDensity}{$(\bm\phi)$}\;
    $\hat N \leftarrow \dfrac{1}{e^{\big(\log(\rho') - \log(\rho)\big)}-1}$\tcp*[f]{Pseudo-count}\;
    \BlankLine
    \Return $\dfrac{\beta}{\sqrt{\hat N}}$\tcp*[f]{Exploration Bonus}\;
}
\end{algorithm}

\subsection{Action Selection}
\label{sec:action_selection}

In the early stages of the project, our agent was facing some inexplicable issues. It had really slow learning progress, and was getting stuck with a single action for long periods of time. Fortunately, we had rich logs that helped us in identifying a pattern to the problem.

We observed that during the initial training cycles, the value predictions from LFA had very high variance due to lack of enough samples. In cases when there was an abnormally high Q-value, our greedy optimistic agent always kept taking the same action over and over again in a loop. We initially thought that, decay in the corresponding exploration bonus, coupled with increase in optimistic estimates for other states would lead to the agent breaking out of the loop. Even though eventually the agent got out of the loop, it happened only after an exorbitantly large number of episodes. From the logs we observed that each TD update only effected a small change, and hence the reason why it took a large number of episodes to overcome abnormally high Q-value.

If we were using $\epsilon$-greedy as the exploration strategy this would not be a problem. With $\epsilon$-greedy, the agent takes more exploratory action in the initial training cycles. Even if LFA produces highly varying Q-values initially, the agent doesn't get stuck for more that a few cycles. Thus, it can be noted that random exploration at the beginning helps stabilize the action-values predicted by LFA.

Our goal is to replace $\epsilon$-greedy with our intrinsically motivated exploration strategy $\bm\phi$-EB. Unfortunately, the removal of $\epsilon$-greedy meant that the agent's policy is now deterministic and has the above debilitating side-effect. In order to solve this crippling issue we experimented two approaches.
\begin{itemize}
    \item \textbf{Combine $\bm\phi$-EB with $\epsilon$-greedy}\\
    A similar problem was reported by \cite{Bellemare2016}. Their solution was to use $\epsilon$-greedy, not as an exploration strategy, but as a tool to introduce stochasticity in the agents policy. During the initial training cycles, when there is high variance from the LFA estimates, taking a purely random action allows the agent to get out of the greedy action loop. In this experiment $\epsilon$-greedy is implemented in the usual way - with probability $\epsilon$ take a random action, and a greedy optimistic action otherwise. Algorithm \ref{alg:impl_e_greedy} presents the implementation.
    \item \textbf{Combine $\bm\phi$-EB with Boltzmann distributed action selection}\\
    One motivation for our research is to make sure that the agent does not take purely random actions. The approach from \cite{Bellemare2016} described above introduces purely random actions. We present an alternate approach which introduces stochasticity but in a directed manner.
    
    We split our optimistic $Q$ functions into two functions, $Q^{\mathcal{E}}$ and $Q^{\mathcal{I}}$. $Q^{\mathcal{E}}$ is trained using the extrinsic reward, whereas $Q^{\mathcal{I}}$ is trained on the exploration bonus from $\bm\phi$-EB\footnote{When using LFA, training is done on the LFA parameters. Therefore we essentially maintain two sets of parameters, $\bm\theta^{\mathcal{E}}$ and $\bm\theta^{\mathcal{I}}$}. The motivation here is that we now have a value-function $Q^{\mathcal{I}}$ that directs the exploratory actions of the agent. For action selection we construct the optimistic value function as the summed value function $Q=Q^{\mathcal{E}} + Q^{\mathcal{I}}$. During action selection, with probability $(1-\epsilon)$ the agent takes the action that is greedy with respect to $Q$, otherwise the agent takes a Boltzmann distributed random action. The Boltzmann distribution is constructed from the $Q^{\mathcal{I}}$ values using the \texttt{discrete\_distribution} standard library. Hence the selected random action is more likely to be an action that has higher exploratory value. Algorithm \ref{alg:impl_boltz_greedy} presents implementation for this approach.
\end{itemize}
Theoretically, the only difference between the above two approaches is the action selection process during exploration. The first approach takes a uniformly random action, whereas the second one takes a Boltzmann-distributed random action. Therefore during the implementation of the learning algorithm we implement the second approach, and swap the action selection process with the first for experimentation.
\begin{algorithm}[H]
\caption{Action Selection: $\epsilon$-Greedy}
\label{alg:impl_e_greedy}
\SetAlgoLined
\DontPrintSemicolon
\SetKwProg{Fn}{function}{}{end function}
\Fn{\textsc{NextAction}{$(Q)$}} {
    \tcc*[l]{$Q$ contains $Q$-values $\forall a\in\mathcal{A}$ for some state.}
    $a = \arg\max\limits_{a\in\mathcal{A}}Q(a)$\;
    \BlankLine
    \tcc*[l]{rand(0,1) generates random number between 0,1}
    \If{\texttt{rand(0,1)} $<\epsilon$} {
    \BlankLine
        \tcc*[l]{randInt(1,x) generates a uniformly random integer between 1,x}
        \BlankLine
        $i=$ \texttt{randInt}(1,$|\mathcal{A}|$)\;
        \Return $a_i$ \tcp*[f]{random action}\;
    }
    \Return $a$ \tcp*[f]{Greedy Optimistic action}\;
}
\end{algorithm}

\begin{algorithm}[H]
\caption{Action Selection: Boltzmann Distributed}
\label{alg:impl_boltz_greedy}
\SetAlgoLined
\DontPrintSemicolon
\SetKwProg{Fn}{function}{}{end function}
\Fn{\textsc{NextActionBoltz}{$(Q^{\mathcal{E}}, Q^{\mathcal{I}})$}} {
    \tcc*[l]{$Q^{\mathcal{E}}, Q^{\mathcal{I}}$ contains $Q$-values $\forall a\in\mathcal{A}$ for some state.}
    $a = \arg\max\limits_{a\in\mathcal{A}}\big\{Q^{\mathcal{E}}(a)+ Q^{\mathcal{I}}(a)\big\}$\;
    \BlankLine
    \tcc*[l]{rand(0,1) generates random number between 0,1}
    \If{\texttt{rand(0,1)} $<\epsilon$} {
    \BlankLine
    \tcc*[l]{boltzDistInt(W,1,x) generates an integer between 1,x that is Boltzmann distributed according to W}
    \BlankLine
        $i=$ \texttt{boltzDistInt}($Q^{\mathcal{I}}, 1,|\mathcal{A}|$)\;
        
        \Return $a_i$ \tcp*[f]{Boltzmann distributed random action}\;
    }
    \Return $a$ \tcp*[f]{Greedy Optimistic action}\;
}
\end{algorithm}

\subsection{SARSA(\texorpdfstring{$\lambda$}{lambda})+\texorpdfstring{$\bm\phi$}{phi}-EB}

Now that we have all the modules necessary for learning, we present the implementation for our agent in Algorithm \ref{alg:impl_rl_sarsa_eb}.

\begin{algorithm}[H]
\caption{Reinforcement Learning with SARSA$(\lambda)$ and $\bm\phi$-EB exploration}
\label{alg:impl_rl_sarsa_eb}
\SetAlgoLined
\DontPrintSemicolon
\KwIn{Feature Map $\bm\phi:\mathcal{S}\rightarrow \mathbb{R}^M$; Training Horizon $t_{end}$}
$t\leftarrow0$\;
\tcc*[h]{Each $\bm\theta$ are an $|\mathcal{A}|\times M_t$ matrix.}\;
Initialize arbitrary $\bm\theta_t^{\mathcal{E}},\bm\theta_t^{\mathcal{I}}$
$s_t\leftarrow$ Initial state\;
$Q_t^{\mathcal{E}} \leftarrow \bm\theta_t^{\mathcal{E}}\bm\phi(s_t)$\tcp*[f]{Vector containing $Q^\mathcal{E}$-values $\forall a\in\mathcal{A}$}\;
$Q_t^{\mathcal{I}} \leftarrow \bm\theta_t^{\mathcal{I}}\bm\phi(s_t)$\tcp*[f]{Vector containing $Q^\mathcal{I}$-values $\forall a\in\mathcal{A}$}\;
$a_t\leftarrow$ \textsc{NextActionBoltz}{$(Q_t^{\mathcal{E}}, Q_t^{\mathcal{I}}, s_t)$}\;
\While{$t<t_{end}$}{
\tcc*[l]{Re-estimate $Q$-values with updated $\bm\theta$ values.}
$Q_t^{\mathcal{E}} \leftarrow \bm\theta_t^{\mathcal{E}}\bm\phi(s_t)$\;
$Q_t^{\mathcal{I}} \leftarrow \bm\theta_t^{\mathcal{I}}\bm\phi(s_t)$\;
\BlankLine
$r_{t+1},s_{t+1} \leftarrow$ \textsc{Act}{$(a_t)$} \tcp*[h]{Perform action in ALE.}\;
$\mathcal{R}_t^\phi\leftarrow$ \textsc{ExplorationBonus}{$(s_t)$}\tcp*[h]{Compute Intrinsic reward.}\;
\BlankLine
\tcc*[l]{Predict next state $Q$-values.}
$Q_{t+1}^{\mathcal{E}} \leftarrow \bm\theta_t^{\mathcal{E}}\bm\phi(s_{t+1})$\;
$Q_{t+1}^{\mathcal{I}} \leftarrow \bm\theta_t^{\mathcal{I}}\bm\phi(s_{t+1})$\;
\BlankLine
\tcc*[l]{Boltzmann distributed action selection.}
$a_{t+1}\leftarrow$ \textsc{NextActionBoltz}{$(Q_{t+1}^{\mathcal{E}},
Q_{t+1}^{\mathcal{I}})$}\;\label{alg:line_action_sel}
\tcc*[l]{Alternatively: $\epsilon$-greedy action selection.}
\tcp*[l]{$a_{t+1}\leftarrow$ \textsc{NextAction}{$(Q_{t+1}^{\mathcal{E}} + Q_{t+1}^{\mathcal{I}})$}}
\BlankLine
\tcc*[l]{TD update}
$\delta_{t+1}^{\mathcal{E}} = r_{t+1} + \gamma Q_{t+1}^{\mathcal{E}}(a_{t+1}) - Q_t^{\mathcal{E}}(a_t)$\;
$\delta_{t+1}^{\mathcal{I}} = R_{t}^{\phi} + \gamma Q_{t+1}^{\mathcal{I}}(a_{t+1}) - Q_t^{\mathcal{I}}(a_t)$\;
$\bm\theta_{t+1}^{\mathcal{E}}\leftarrow\bm\theta_t^{\mathcal{E}} + \alpha\delta_{t+1}^{\mathcal{E}}\mathbb{I}_{|\mathcal{A}|\times M_t}$\;
$\bm\theta_{t+1}^{\mathcal{I}}\leftarrow\bm\theta_t^{\mathcal{I}} + \alpha\delta_{t+1}^{\mathcal{I}}\mathbb{I}_{|\mathcal{A}|\times M_t}$\;
$t\leftarrow t+1$\;
}
\Return $\bm\theta_{t_{end}}^{\mathcal{E}},\bm\theta_{t_{end}}^{\mathcal{I}}$\;
\footnotetext{$M_t$ is the number feature observed til timestep $t$. We have removed the details regarding eligibility traces for brevity and clarity.}
\end{algorithm}

Algorithm \ref{alg:impl_rl_sarsa_eb} presents the final version of the algorithm that we have implemented, and for which empirical results are presented. In the algorithm shown we use the Boltzmann distributed action selection approach. We can disable Line \ref{alg:line_action_sel} and enable the two comments below it to use the $\epsilon$-greedy action selection approach.

In the next chapter we discuss the experimental evaluation framework we used to perform empirical evaluation. Then we showcase the state-of-the-art results that our algorithms enjoys.

\chapter{Empirical Evaluation}
\label{cha:result}

\epigraph{\textit{'What can be asserted without evidence can also be dismissed without evidence.'}}{Christopher Hitchens}

Empirical evidence is one of the fundamental requirements for validating any scientific hypothesis. In order to validate the efficacy of our exploration algorithm, this chapter showcases empirical results that represent a significant improvement over existing algorithms.

In \cref{sec:eval_framework} we talk about the evaluation platform and the feature set that we used to evaluate our exploration strategy. \cref{sec:ale} introduces the Arcade Learning Environment (ALE) as our evaluation platform. We provides justification for choosing ALE as an environment for our agent. Further, in \cref{sec:feature_set} we introduce the Blob-PROST feature set, and the benefits our agent enjoys from using it.

In \Cref{sec:eval_methd} we provide the necessary foundations needed to evaluate Algorithm~\ref{alg:impl_rl_sarsa_eb}. We discuss the aspects that need to be considered in choosing a particular game for evaluation. Further, we talk about the parameters for empirical evaluation, such are number of trial, training frames, etc.

Lastly, in \Cref{sec:emp_results} we discuss the results of our various experiments. We compare the two action selection process discussed in \cref{sec:action_selection} and compare their empirical performance. Then we compare the learning performance of our agent with SARSA$(\lambda)$+$\epsilon$-greedy. Finally, we compare the evaluation scores for our agent with other leading algorithms.

\section{Evaluation Framework}
\label{sec:eval_framework}

\subsection{Arcade Learning Algorithm (ALE)}
\label{sec:ale}

The Arcade Learning Environment (ALE)~\citep{Bellemare2013} is a software framework that interfaces with the Stella emulator~\citep{MottBradfordWandTeam1996} for the Atari 2600 games \citep{Montfort2009}. The Atari 2600 platform contains hundreds of games that vary in many aspects of game-playing such as sports, puzzle, action, adventure, arcade, strategy etc. (\cref{fig:ale_games}). Some of the games are quite challenging for human players~\citep{Bellemare2013}. Due to the diverse nature of the games, a learning algorithm that can play the entire gamut of the Atari 2600 games can be considered to be generally competent. The goal of the ALE framework is to provide a platform for AI researchers to test their learning algorithm for general competence, share empirical data with the research community, and further the goal of achieving artificial general intelligence~\citep{Bellemare2013}.

\begin{figure}[H]
\label{fig:ale_games}
  \centering
  \includegraphics[scale=0.26]{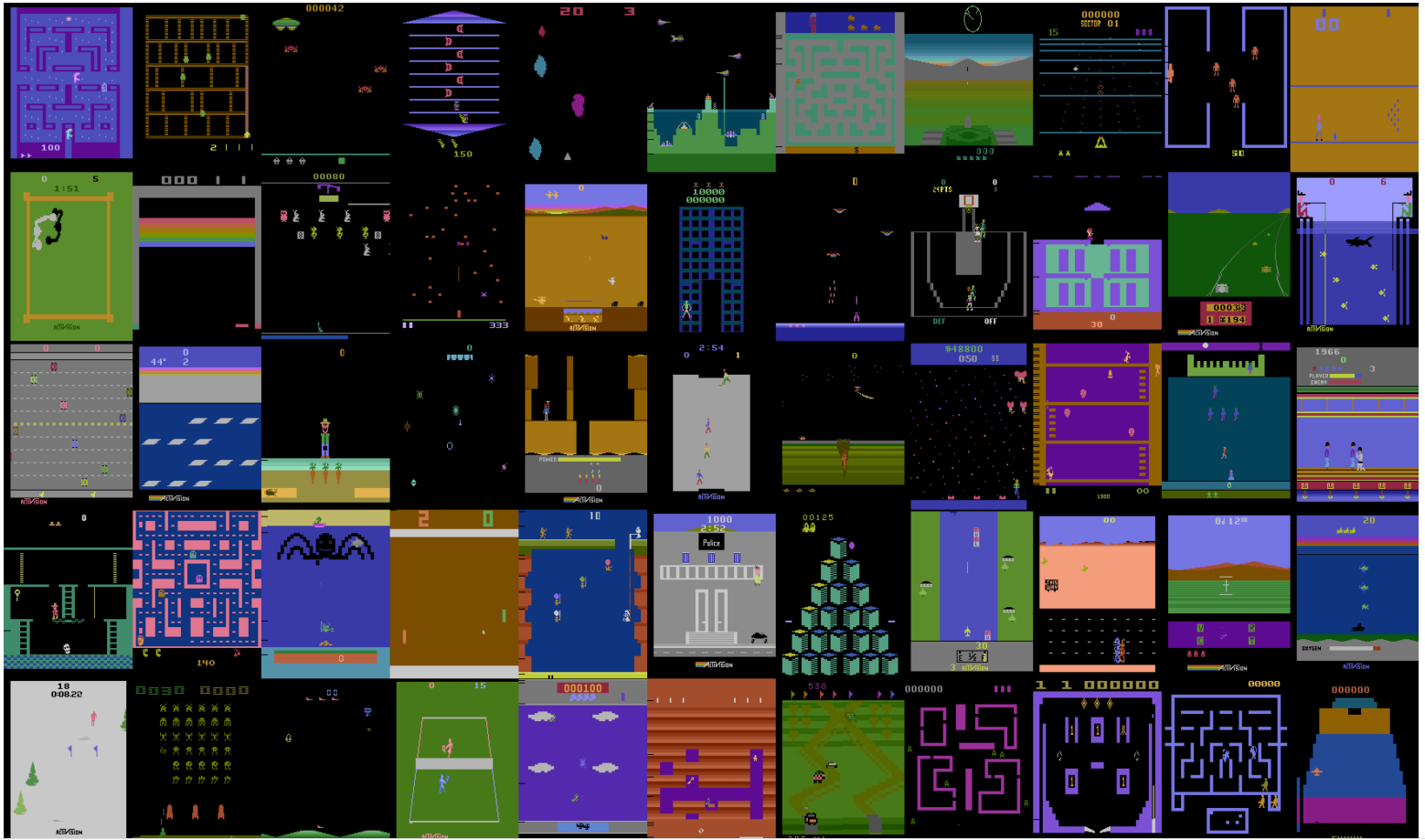}
  \caption[Game screens from 55 Atari 2600 games.]{Game screens from 55 Atari 2600 games~\citep{Defazio2014}.}
\end{figure}

\begin{figure}[H]
\label{fig:ale_rl}
  \centering
  \includegraphics[scale=0.55]{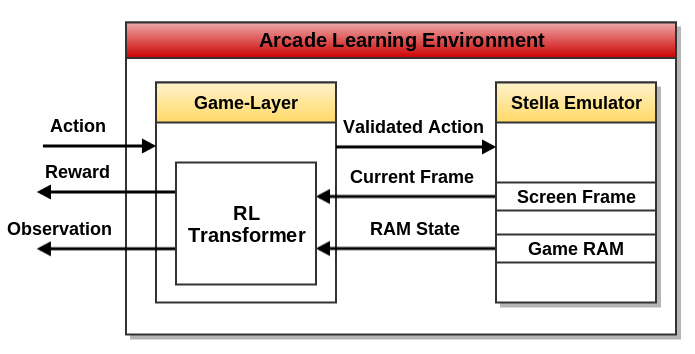}
  \caption{High level working of ALE for an RL algorithm.}
\end{figure}

The ALE contains a game-layer to facilitate reinforcement learning. The game-layer takes in the action from the agent and validates if it is one of the predefined $18$ discrete actions. The game-layer sends the validated action to the Stella emulator which performs the action on the chosen Atari 2600 game. The resulting screen frame and the RAM state is sent to the game-layer by the emulator. Each screen frame is a $160 \times 210$ 2D array, with each element representing a $7$-bit pixel. The game-layer analyses the frame and RAM state to identify the game score. It then transform the game score into an appropriate reward signal expected by an RL agent. Based on the configuration settings, the game-layer returns the frame array and/or RAM state as the observation.

After the advent of the ALE, the majority of reinforcement learning publications have used ALE to provide empirical evidence~\citep{Mnih2013,Mnih2015,Stadie2015,Osband2016a,Bellemare2016,Tang2016,Ostrovski2017}. Now, ALE enjoys the status as the standard test bed for testing RL algorithms. Therefore, in order to compare and contrast our results with existing research, it is crucial that we choose ALE as our evaluation platform.

\subsection{Blob-PROST Feature Set}
\label{sec:feature_set}

We observed breakthrough performance by DQN to achieve human level performance on majority of ALE games~\citep{Mnih2015}. Subsequently~\cite{Liang2015} did a systematic study to analyse the factors that resulted in such a dramatic increase in performance. As part of this study, they created feature sets that incorporated key representational biases encoded by DQN\@. One such data set is called Blob-PROST\@.

PROST in Blob-PROST stands for Pairwise Relative Offset in Space and Time. They argue that in most games absolute position of objects are not as important as their relative distance. Therefore, by taking into account the relative distance between two objects on screen, they are able to encode information like ``there is a green pixel 5 tiles above a blue pixel''. To encode the movements of objects, they take the relative distance of an object in the current frame to the object five frames in the past.

In an Atari game screen we can assume that there are many blocks of contiguous pixels with the same color. This is a common continuity assumption that is typically made in the context of computer vision.~\cite{Liang2015} exploits this assumption and calls such blocks \textit{blobs}. From a high level,~\cite{Liang2015} first pre-process a frame to find blobs, and then find features based on pairwise relative distances. We refer the reader to~\cite{Liang2015} for a full treatment on the construction of the feature set. 

The Blob-PROST feature set contains a total of 114,702,400 binary features. Even though there are a large number of potential features, due to the sparsity of blobs, most of the features are never generated. Also, given a specific game only a relatively small number of features would be observed. Therefore, using Blob-PROST with LFA is far more computationally efficient that DQN\@. Moreover, empirical results from~\cite{Liang2015} suggest that the fixed representations constructed in Blob-PROST has the same quality as those learned by DQN\@.

\section{Empirical Evaluation}
\label{sec:eval_methd}

The evaluation methodology was designed to investigate and answer the following research questions:
\begin{itemize}
    \item Does the novelty measure generalize state visit-counts when the generalization is performed in the same space (feature space) as the generalization regarding value?
    \item Is there performance improvement over different kinds of environments?
    \item How does the performance of our algorithm fare when compared to state-of-the-art Deep RL algorithms?
\end{itemize}

\subsection{Evaluation Methodology}
\label{sec:eval_meth}

We evaluate Algorithm~\ref{alg:impl_rl_sarsa_eb} in the ALE using the Blob-PROST feature set. We evaluate our algorithm on a subset of the games that are most relevant to the problem of efficient directed exploration. An important factor that determines our selection is the time and computational resources required to perform the evaluation. Finally, we perform hyperparameter sweeps to appropriately tune our algorithm.

\begin{figure}[H]
\label{fig:ale_five_games}
  \centering
  \includegraphics[scale=0.3855]{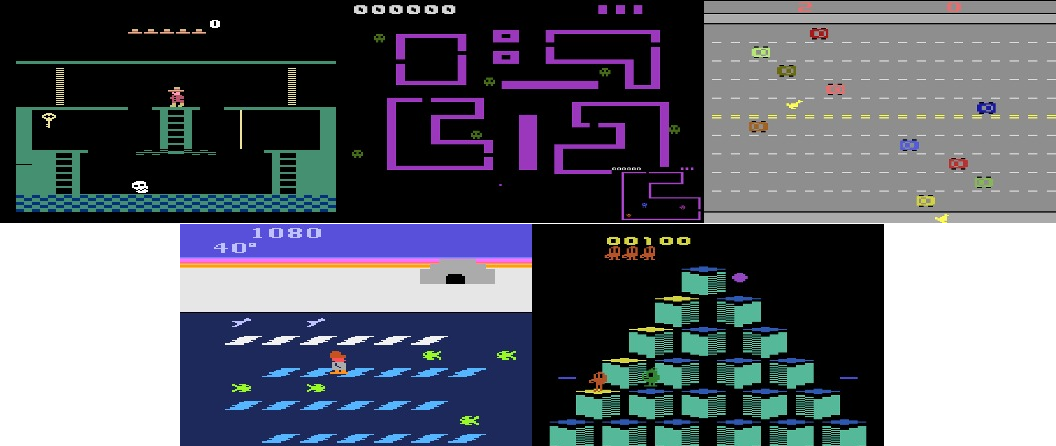}
  \caption{Five games in which exploration is both difficult and crucial to performance. From top left: Montezuma's Revenge, Venture, Freeway, Frostbite, and Q*Bert.}
\end{figure}
\subsubsection*{Choosing Evaluation Games}
ALE contains many games which vary in the degree to which exploration is difficult~\citep{Bellemare2016}. At the lower end of the difficulty spectrum are games for which undirected exploration ($\epsilon$-greedy) is sufficient to learn a good policy~\citep{Mnih2015}. We evaluate our algorithm on hard games, \ie{} games where $\epsilon$-greedy fails to improve substantially on a random policy. Hence we focus our evaluation on games that are classified as hard in the taxonomy provided by~\cite{Bellemare2016}. In their taxonomy they further split hard exploration games into games that have sparse and dense rewards. 

In dense-reward games, our RL agent can expect rewards, on average, every few cycles. Dense-reward games are easier for an RL agent because policy iteration techniques such as SARSA generally require regular feedback on the agent's policy to learn; this is known as the issue of temporal credit assignment~\citep{Sutton1998}.

On the other hand, sparse-reward games only dispense rewards infrequently. In this setting, the agent must often perform long sequences of actions in the correct order in order to receive a reward\footnote{Classic examples of such games are the board games Chess and Go; rewards are only dispensed once, at the very end of the game.}. Without a good exploration strategy, stumbling upon a productive sequence of action is very challenging. Therefore, our main focus will be on sparse-reward games in which exploration is difficult.

We compare our performance with that of  state-of-the-art Deep RL algorithms reported in the literature, most of which are variants of DQN~\citep{VanHasselt2015}. Recall that we use LFA for value prediction, while DQN uses neural networks. As discussed previously, the Blob-PROST feature set that we selected for LFA has the property that it closely models the representations learned by DQN~\citep{Liang2015}. Since the Blob-PROST feature set has this property, it is appropriate to compare our $\bm\phi$-EB algorithm to the other algorithms that use DQN\@. Now from the empirical data presented in~\cite{Liang2015}, we look for games that perform similar to DQN, and meet the exploration difficult criterion outlined above.

With these in mind, we choose the following five hard exploration games to evaluate our exploration strategy.
\begin{itemize}
    \item Sparse Reward Games (\cref{fig:ale_five_games}, Top Layer)
    \begin{itemize}
        \item \textsc{Montezuma's Revenge}
        \item \textsc{Venture}                
        \item \textsc{Freeway}
    \end{itemize}
    \item Dense Reward Games (\cref{fig:ale_five_games}, Bottom Layer)
    \begin{itemize}
        \item \textsc{Frostbite}
        \item \textsc{Q*bert}
    \end{itemize}
\end{itemize}

\subsubsection*{Computational Roadblocks}
The computational requirement for games in the ALE are very demanding, especially in the absence of graphical processing units (GPUs), which excel at the dense linear algebra computations common in  vision-related tasks~\citep{Liang2015}. Given the high-dimensional nature of the problem, agents must be trained for several days on end to obtain satisfactory performance. Frostbite and Q*bert were especially computationally intensive, and we had to train for several weeks to obtain sufficient data. Given our limited time and computational resources, we had to place the following constraints on our evaluation.

\begin{itemize}
    \item In \cref{sec:fvd_impl_det} we remarked that due to sparse nature of the feature vector our algorithm runs in $O(M_t)$ where $M_t$ is the number of unique features observed till time $t$. This means that different games run at different speeds. We trained our algorithm on all games for 100 million frames except for Q*bert which was trained only for 80 million frames.
    \item Our main focus is to showcase performance gains in sparse reward hard exploration games. Therefore we focused bulk of our computational resources into running multiple trial for Montezuma's Revenge and Venture. This is where our algorithm leads other state-of-the-art exploration strategies.
\end{itemize}

\subsubsection*{Tuning the hyper-parameter $\beta$}
\label{sec:beta_sweep}
We performed independent evaluation of all the chosen games for discrete values of $\beta$ in the range $[0.0001, 0.5]$. Recall that $\beta$ is a parameter that controls the magnitude of the exploration bonus. We observed that $\beta=0.05$ is the best performing value for all the games except for Freeway. Recall that because of the nature of the game, there is a large number of unique Blob-PROST features active. If $\beta$ is high enough, the chicken just remains stationary and receives novelty rewards for observing all the changes in the traffic. When we set $\beta=0.035$, our agent performed much better and delivered comparable results with the baseline algorithm.

\subsubsection*{Training Methodology}

Training and evaluation of learning algorithms in high-dimensional spaces is computationally demanding. Due to the constantly evolving nature of the field, there is no general consensus on how many cycles is required to train an agent. Some of the major exploration algorithms published recently report training till 200 million frames. Due to limited time and computational resources, this amount of training is not feasible for us.

We perform empirical evaluation of Montezuma's Revenge and Venture for five independent trials, and two independent trials for Frostbite, Freeway, and Q*bert. With the exception of Q*bert, all agents are trained for 100 Million frames, and then evaluated for 500 episodes. The result for Q*bert is reported after training for 80 million frames with subsequent 500 episodes of training.

We use the average score per episode, which is a common metric used to report scores~\citep{Bellemare2013,Mnih2015,Bellemare2016}.

\subsection{Sparse Reward Games}
\label{sec:sparce_rew}
\subsubsection*{Montezuma's Revenge}
Montezuma's Revenge is widely regarded as one of the most difficult games in the Atari 2600 suite. Learning algorithms typically suffer here due to the problem of long term credit assignment and sparse rewards. For example, DQN with $\epsilon$-greedy exploration achieves a score of $0$ after training for 200 million frames~\citep{Mnih2015}. In order to get the very first reward of the game the agent must climb down a ladder, jump onto the pole, jump onto a raised platform, climb down a ladder, walk left and jump to avoid an enemy, climb another ladder, and finally obtain a golden key. This long sequence of complex actions is required to simply achieve the first reward in the first of 24 rooms, arranged in a pyramid structure (\cref{fig:pyramid}). It is evident from the complexity of the game that a random exploration strategy will fail miserably. The challenges posed by this game make the game central to our evaluation.
\begin{figure}[H]
\label{fig:pyramid}
  \centering
  \includegraphics[scale=0.268]{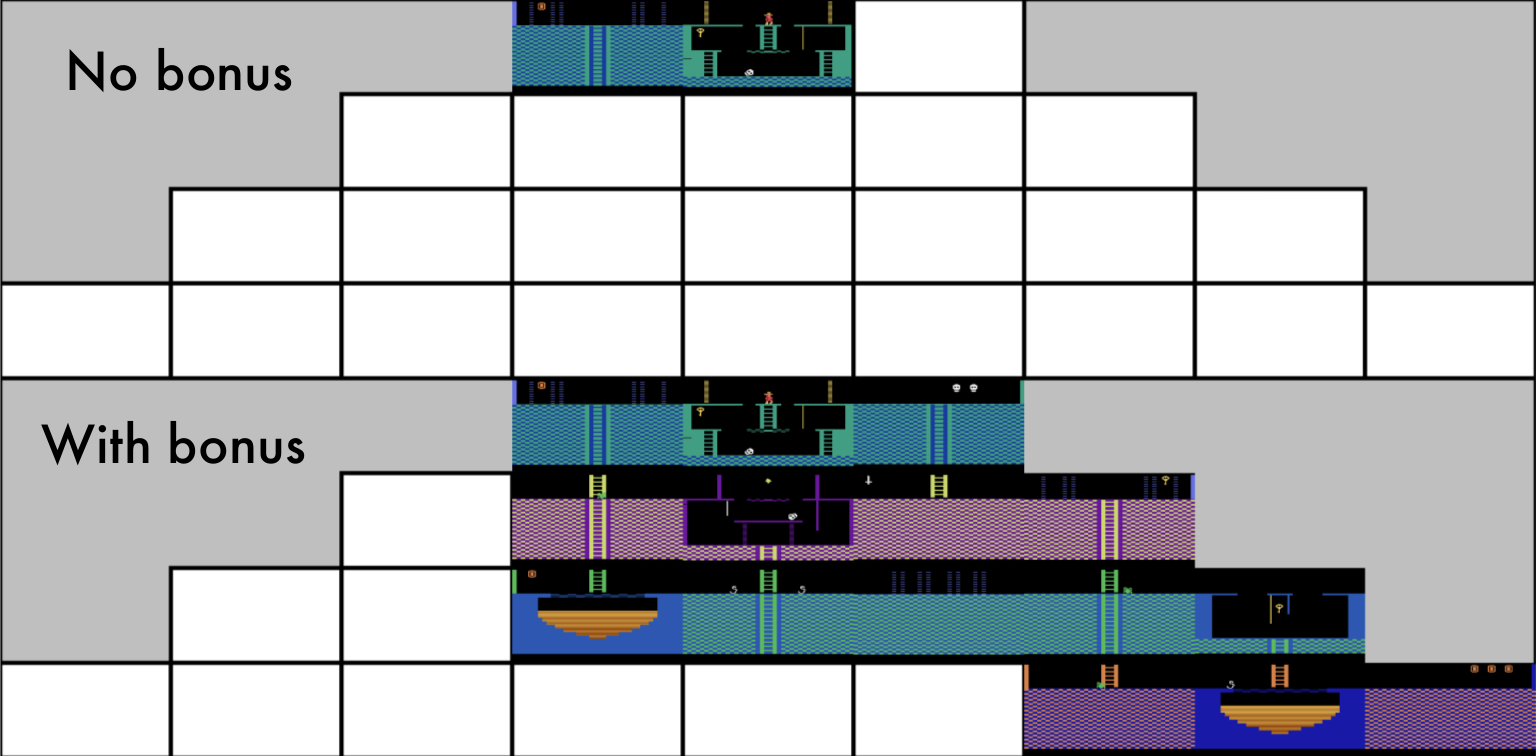}
  \caption{Montezuma's Revenge: Rooms visited by undirected exploration (DQN+$\epsilon$-greedy, above) vs.\ directed exploration (DQN+CTS-EB, below)~\citep{Bellemare2016}.}
\end{figure}

Our agent showed steady learning progress and was relatively quick to learn the initial sequence of actions needed to get the golden key. Analysis from the logs showed increasing novelty measure associated with novel events. For example, one of the rooms had a dead-end for one side, but there was a reward there. When our agent reached the dead-end and received the reward once, it kept getting stuck in that region for a few episodes. During the episodes where the agent was stuck, we could observe the novelty values of the other end slowly rising. After a few episodes the agent started going in the opposite direction based on the novelty rewards it was getting, and proceeded to explore further.

In Montezuma's revenge, the total number of rooms visited is a good measure of exploration efficiency~\citep{Bellemare2016}. Prior to the work of~\cite{Bellemare2016} no agent had visited more that 3 rooms without domain specific tailoring.~\cite{Bellemare2016}'s agent visited 15 room while our agent visited 14 rooms in total (\cref{fig:pyramid}). Both~\cite{Bellemare2016}'s and our agent enjoy a peak score of \textbf{6600}, which is the highest reported score.

\subsubsection*{Venture}
Venture is another hard, sparse reward exploration game with complex visual representations. \cref{fig:venture_img} represents the two different visual states of venture. In \cref{fig:vent_outer} the agent is depicted by the small tiny pink pixel towards the bottom middle of the screen. When the agent is in the outer level it is powerless. The only way to stay alive is to navigate the maze and/or entering one of the rooms to avoid the green goblins. Once the agent enters the room, the screen changes to \cref{fig:vent_inner}. Now our agent is transformed into a smiley face with the ability to fire a projectile weapon. Here the agent is chased by evil blue robots which can be destroyed by the agent's weapon. The small tiny cup-like structure in the bottom left corner of the room is the goal and the first reward. Destroying the blue monsters does not result in any reward. The game presents some intricate dynamics, with a large state-space to explore, and so is difficult for a naive learning algorithm. Therefore Venture is a must-have game in our evaluation roster.

\begin{figure}[H]
    \begin{minipage}{0.48\textwidth}
        \centering
        \includegraphics[scale=0.305]{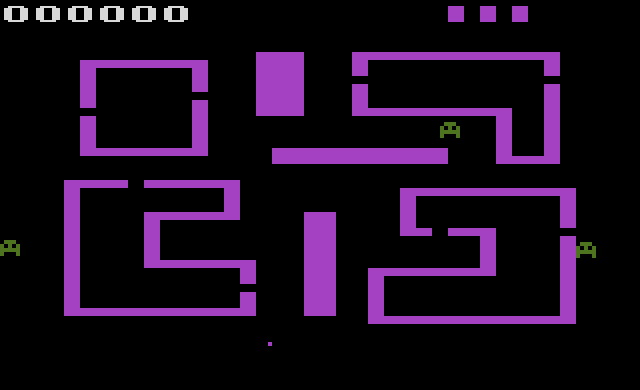}
        \caption{Venture Outer Level}\label{fig:vent_outer}
    \end{minipage}\hfill
    \begin{minipage}{0.48\textwidth}
        \centering
        \includegraphics[scale=0.305]{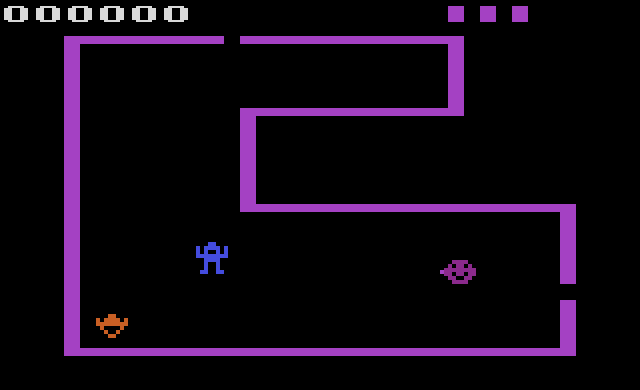}
        \caption{Venture Inner Level}\label{fig:vent_inner}
    \end{minipage}        
    \caption{Two visual states of venture (Atari 2600)}
\label{fig:venture_img}
\end{figure}

The agent initially moves around the black region to the bottom of the screen where it starts out (\cref{fig:vent_outer}). The novelty of the agent with empty black space around it quickly reduces. We observed that the agent starts to hug the outer walls of the rooms. The room walls together with the agents are novel features. Then the agent moves along the room wall towards the entry of the room. When the agent enters the room the screen now transforms as shown in \cref{fig:vent_inner}. We later observe that, with subsequent visits to the lower room our agent learned to shoot and kill the blue robots and get the  reward. Our agent achieves substantial improvement over~\cite{Bellemare2016}. Even though they do not report their evaluation score for Venture,~\cite{Ostrovski2017} reports the score of DQN+CTS-EB as $82.2$, whereas our evaluation score is $1169.2$.

\subsubsection*{Freeway}
\begin{figure}[H]
\label{fig:freeway_img}
  \centering
  \includegraphics[scale=0.33]{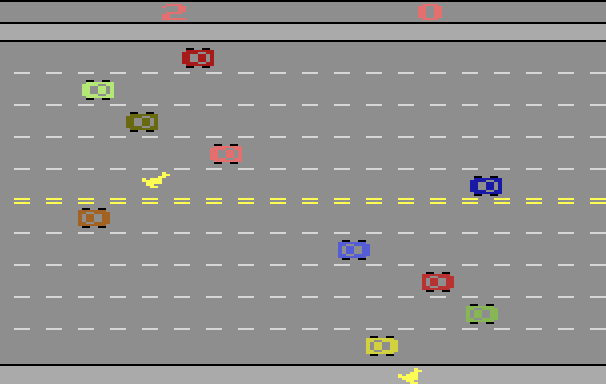}
  \caption{Freeway (Atari2600)}
\end{figure}
Here the agent is a chicken that must cross a busy highway (\cref{fig:ale_five_games} Top layer, third game). The agent receives the reward only when it reaches the other end of the road. Being hit by a car sets you back some positions down the road. Even though this is not a complex game, it is still a hard exploration game. In freeway, cars are always moving from left to right in every frame. Therefore, there is a large number of active unique Blob-PROST features.

When tuning the hyper-parameter $\beta$ we talked about the large number of unique Blob-PROST features created due to the constant movement of traffic. These changing unique features constantly floods the agent with high novelty rewards. Therefore the agent is willing to just stand and observe the traffic. When we reduced the exploration coefficient $\beta$, the agent's extrinsic reward is no longer overwhelmed by the novelty reward. Also, due to the initial optimism, the agent is encouraged to move upwards. Once the agent manages to cross the road once, it reinforces the up action and there the agent starts learning faster.

\subsection{Dense Reward Games}
\subsubsection*{Q*bert}
\begin{figure}[H]
\label{fig:qbert_img}
  \centering
  \includegraphics[scale=0.55]{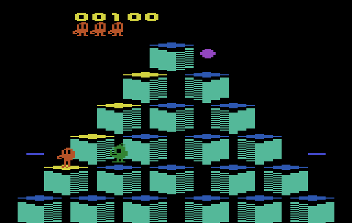}
  \caption{Qbert (Atari2600)}
\end{figure}
In Q*Bert, the agent stands on a pyramid of cubes (\cref{fig:qbert_img}). The goal of the agent is to jump on all the cubes without falling off the edge or being captured by an adversary. When the agent has highlighted all the cubes by jumping on them at least once, the level is cleared. The game has multiple levels. On higher levels, the task still remain the same, but the enemies become smarter, making it increasingly difficult to accomplish the task while avoiding capture. As discussed in \cref{sec:conf_novelty} the only other difference between levels is the choice of color scheme.


\subsubsection*{Frostbite}
\begin{figure}[H]
\label{fig:frostbite_img}
  \centering
  \includegraphics[scale=0.55]{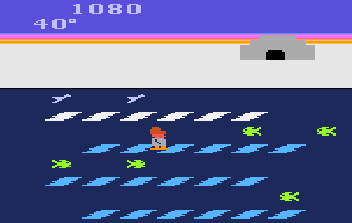}
  \caption{Frostbite (Atari2600)}
\end{figure}
In Frostbite, the agent is an Eskimo jumping up and down on an ice platform which magically results in the building of an igloo (\cref{fig:frostbite_img}). Each time the agent jumps on a pure white ice platform, the agent receives a reward. The agents has to consistently do this by avoiding obstacles/dangers, and picking up bonus points. Once the igloo has been built the agent need to perform a level-end move by entering into the igloo. Since there are multiple ways of maximizing the score, and the existence of a level-end move makes it a hard exploration game. Our main aim with this game is to confirm that our exploration strategy improves upon the baseline performance.

\section{Results}
\label{sec:emp_results}
From here on we denote our agent SARSA$(\lambda)$+$\bm\phi$-EB as SARSA-$\phi$-EB and SARSA$(\lambda)$+$\epsilon$-greedy as SARSA-$\epsilon$ for notational brevity.

\subsection{Boltzmann vs.\ \texorpdfstring{$\epsilon$}{epsilon}-greedy Action Selection}

\begin{figure}[H]
\begin{center}
\includegraphics[scale=0.46]{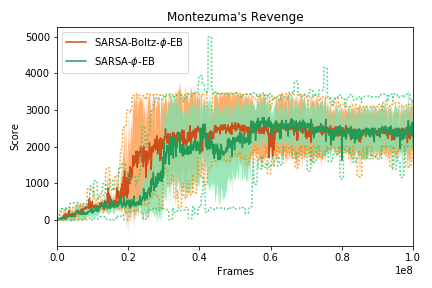}\quad
\includegraphics[scale=0.46]{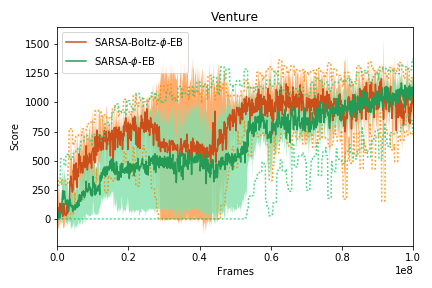}\quad \\
\end{center}
\caption{Average training score for Action Selection using Boltzmann vs.\ $\epsilon$-greedy. Shaded regions describe one standard deviation. Dashed lines represent min/max scores.}
\label{fig:plot_boltz}
\end{figure}

Recall that in \cref{sec:action_selection} we discussed the problem of action selection and presented two methods that introduce stochasticity into the agent's policy. Here we compare two agents differing only in the action selection functionality. SARSA-Boltz-$\phi$-EB introduces stochasticity  via a Boltzmann-distributed random action, whereas SARSA-$\phi$-EB does via $\epsilon$-greedy. 

\cref{fig:plot_boltz} shows the learning curve for the two agent evaluated for the games Montezuma's Revenge and Venture. From the plots we can observe that during the initial training period Boltzmann-distributed action selection has an advantage over epsilon-greedy. In the long run both have essentially the same average score. This behavior is expected since $\epsilon$ is annealed with training cycles, and hence with time the agent takes fewer purely exploratory actions.

What we hoped to see in this experiment is for the Boltzmann distributed action selection to send the agent into a steeper learning trajectory. Unfortunately, this experiment concludes that there is no long term gain to having a Boltzmann-distributed action selection process.

\subsection{Comparison with \texorpdfstring{$\epsilon$}{epsilon}-greedy}
\subsubsection*{Overview}

\cref{fig:plot_e_vs_phi} shows the comparison between learning curves for our agent SARSA-$\phi$-EB, and the benchmark implementation SARSA-$\epsilon$. The plots clearly illustrates that SARSA-$\phi$-EB significantly outperforms SARSA-$\epsilon$ on both Montezuma's Revenge and Venture; two of hardest exploration games in the ALE\@. In Q*bert and Frostbite  SARSA-$\phi$-EB consistently outperforms SARSA-$\epsilon$, but not by a huge margin. Frequent rewards from these games give a constant feedback to SARSA-$\epsilon$, helping it to chart a positive learning path. With $\beta=0.05$ freeway fails  to obtain any score, but with $\beta=0.035$ SARSA-$\phi$-EB achieved a marginally better performance that the baseline SARSA-$\epsilon$.
\begin{figure}[H]
\begin{center}
\includegraphics[scale=0.55]{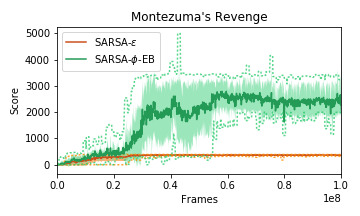}\quad
\includegraphics[scale=0.55]{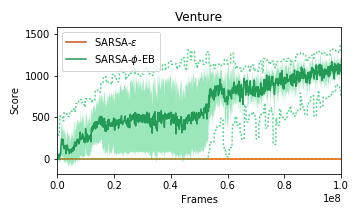}\quad \\
\includegraphics[scale=0.55]{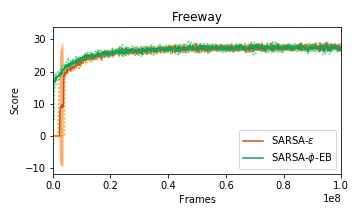}\quad
\includegraphics[scale=0.55]{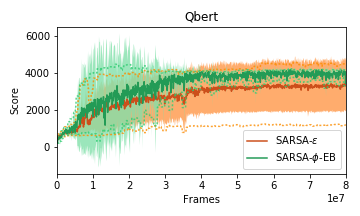}\quad\\
\includegraphics[scale=0.55]{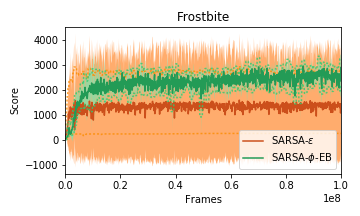}\quad
\end{center}
\caption{Average training score for SARSA-$\phi$-EB vs.\ SARSA-$\epsilon$. Shaded regions describe one standard deviation. Dashed lines represent min/max scores.~\citep{Martin2017}}
\label{fig:plot_e_vs_phi}
\end{figure}

\subsubsection*{Montezuma's Revenge}
After an average of 20 Million frames the policy of SARSA-$\epsilon$ converges. It learns the policy to consistently get the golden key, buy rarely leaves the room. Even when SARSA-$\epsilon$ leaves the room it is always to the left room and gets immediately killed by the laser beams. SARSA-$\epsilon$ never learns how to get through the laser beams.
SARSA-$\phi$-EB performs remarkably well, visiting 14 rooms in total and observing a peak score of 6600 around the 40 Million frame mark. SARSA-$\phi$-EB learns to go through laser doors by timing the movements perfectly, avoid dead-ends, duck and jump over skulls, etc. 

\subsection*{Venture}
SARSA-$\epsilon$ fails to score any points while SARSA-$\phi$-EB performs exceptionally well in Venture. It is the most impressive performance gain in this evaluation. SARSA-$\phi$-EB can quickly identify that spending time in the black region is not novel. It is attracted (because of novelty) to the walls of the room quickly. The agent then moves along the wall until it sees a nearby opposite wall or the room entrance. Once the agent starts entering a room consistently, it quickly learns how to attack blue robots and obtains the reward. The positive learning trend suggest that even higher scores are possible with more training.

\subsubsection*{Freeway}
With $\beta=0.05$ our agent SARSA-$\phi$-EB fails to cross the road for reasons previously mentioned in \cref{sec:eval_methd} and \cref{sec:sparce_rew}. But with $\beta=0.035$, SARSA-$\phi$-EB explores quickly initially and received the reward by crossing the road for the first time. Due to the undirected nature of exploration, SARSA-$\epsilon$ has to first random walk across the road to receive a reward. Once both the algorithms receive a reward they perform similarly, with SARSA-$\phi$-EB doing marginally better.\\
\\
Recall that the components of SARSA-$\phi$-EB were designed such that there is only a single point of change from SARSA-$\epsilon$. The only difference between SARSA-$\phi$-EB and SARSA-$\epsilon$ is our exploration module, $\bm\phi$-EB\@. Hence we can conclude that the empirical performance gains is solely due to the introduction of our exploration algorithm. Further we can conclude that $\bm\phi-EB$ is significantly better that $\epsilon$-greedy in sparse reward hard exploration domains, and consistently outperforms $\epsilon$-greedy in all domains.

\subsection{Comparison with Leading Algorithms}
\cref{tbl:results} shows the evaluation score for the final policy learnt by our agent. In order to compare and contrast the efficacy of our algorithm, \cref{tbl:results} also presents the evaluation scores reported by leading RL algorithms. 

Our agent reports an average evaluation score of 2745.4 on Montezuma's Revenge. This score is the second-highest reported score for Montezuma, the leading score reported for DQN+CTS-EB by~\cite{Bellemare2016}. Algorithms such as MP-EB~\citep{Stadie2015}, TRPO-Hash~\citep{Tang2016}, A3C+~\citep{Bellemare2016} does not score more than 200 points in Montezuma's Revenge, despite the presence of an exploration strategy module.

SARSA-$\phi$-EB evaluated on Venture reported a score of 1169.2, the third-highest reported score. Again it outperforms all the other exploration algorithms. \cite{Bellemare2016} does not report their score for Venture. A recent evaluation of DQN+CTS-EB done by \cite{Ostrovski2017} reported that DQN+CTS-EB obtained a score of 82.2 on venture.

It is known that non-linear algorithms perform much better on Q*bert therefore our results is not surprising \citep{Wang2016a}. Frostbite is notoriously slow to train because of the large number of features. Given the dense reward nature of the game and the small training time, we achieve competitive results.\\
\\
Our results for Montezuma and Venture are extremely competitive and can be considered state-of-the-art. The key observation from \cref{tbl:results} is this, if we look at the score of Montezuma and Venture together, ours is the only algorithm that achieve state-of-the-art results in both the games.

\begin{table}[t]
\begin{centering}
{\footnotesize{}}%
\begin{tabular}{cccccc}
 & \textbf{\scriptsize{}Venture} & \textbf{\scriptsize{}Montezuma} & \textbf{\scriptsize{}Freeway} & \textbf{\scriptsize{}Frostbite} & \textbf{\scriptsize{}Q{*}bert}\tabularnewline\addlinespace
\midrule
\midrule 
\textbf{\scriptsize{}Sarsa-$\phi$-EB} & {\scriptsize{}1169.2} & {\scriptsize{}2745.4} & {\scriptsize{}0.0} & {\scriptsize{}2770.1} & {\scriptsize{}4111.8}\tabularnewline
\midrule 
\textbf{\scriptsize{}Sarsa-$\epsilon$} & {\scriptsize{}0.0} & {\scriptsize{}399.5} & {\scriptsize{}29.9} & {\scriptsize{}1394.3} & {\scriptsize{}3895.3}\tabularnewline
\midrule 
\midrule 
\textbf{\scriptsize{}DQN+CTS-EB} & {\scriptsize{}N/A} & {\scriptsize{}3459} & {\scriptsize{}N/A} & {\scriptsize{}N/A} & {\scriptsize{}N/A}\tabularnewline
\midrule 
\textbf{\scriptsize{}A3C+} & {\scriptsize{}0} & {\scriptsize{}142} & {\scriptsize{}27} & {\scriptsize{}507} & {\scriptsize{}15805}\tabularnewline
\midrule 
\textbf{\scriptsize{}MP-EB} & {\scriptsize{}N/A} & {\scriptsize{}0} & {\scriptsize{}12} & {\scriptsize{}380} & {\scriptsize{}N/A}\tabularnewline
\midrule 
\midrule 
\textbf{\scriptsize{}DDQN} & {\scriptsize{}98} & {\scriptsize{}0} & {\scriptsize{}33} & {\scriptsize{}1683} & {\scriptsize{}15088}\tabularnewline
\midrule 
\textbf{\scriptsize{}Dueling} & {\scriptsize{}497} & {\scriptsize{}0} & {\scriptsize{}0} & {\scriptsize{}4672} & {\scriptsize{}19220}\tabularnewline
\midrule 
\textbf{\scriptsize{}DQN-PA} & {\scriptsize{}1172} & {\scriptsize{}0} & {\scriptsize{}33} & {\scriptsize{}3469} & {\scriptsize{}5237}\tabularnewline
\midrule 
\textbf{\scriptsize{}Gorila} & {\scriptsize{}1245} & {\scriptsize{}4} & {\scriptsize{}12} & {\scriptsize{}605} & {\scriptsize{}10816}\tabularnewline
\midrule 
\textbf{\scriptsize{}TRPO } & {\scriptsize{}121} & {\scriptsize{}0} & {\scriptsize{}16} & {\scriptsize{}2869} & {\scriptsize{}7733}\tabularnewline
\midrule 
\textbf{\scriptsize{}TRPO-Hash} & {\scriptsize{}445} & {\scriptsize{}75} & {\scriptsize{}34} & {\scriptsize{}5214} & {\scriptsize{}N/A}\tabularnewline
\bottomrule
\end{tabular}
\par\end{centering}{\footnotesize \par}
\caption{Average evaluation score for leading algorithms. Sarsa-$\phi$-EB
and Sarsa-$\epsilon$ were evaluated after 100M training frames
on all games except Q{*}bert, for which they trained for 80M frames.
All other algorithms
were evaluated after 200M frames.~\citep{Martin2017}}
\label{tbl:results}
\end{table}
\chapter{Conclusion and Future Work}
\label{cha:conc}

In this thesis we have presented the $\bm\phi$-Exploration Bonus ($\bm\phi$-EB) method, a novel approach to perform directed exploration in large problem domains. The algorithm is simple to implement, and is compatible with any value-based RL algorithm that uses Linear Function Approximation (LFA) to predict value. Our method also enjoys lower computational requirements when compared to other leading exploration strategies. Our empirical evaluation demonstrates that measuring novelty in feature space is a simple and effective way to drive efficient exploration on MDPs of practical interest. It also lends support to our hypothesis that defining a novelty measure in feature space is a principled way to generalize state-visit counts to large problems. In contrast to other approaches, measuring novelty in feature space avoids building an exploration-specific state-representation. Instead, our method exploits the task-relevant features that are already being used for value estimation.

There are myriad ways in which this work could be extended, and the problem of efficient exploration in large MDPs is still wide open. A promising direction for future research would be a rigorous empirical comparison of the various generalized count-based algorithms which we have discussed. At present, many of the reported results are not helpful in deciding which is the better approach to measuring novelty or generalizing visit-counts. Different exploration algorithms are presented in conjunction with totally different value-estimation methods, and it can be difficult to discern whether or not the exploration method used is responsible for the quality of the results. More theoretical understanding of the problem of exploration in the high-dimensional setting is also sorely needed, and we hope to build upon the results presented here in future work.



\backmatter

\bibliographystyle{anuthesis}
\bibliography{mendeley}

\printindex

\end{document}